\documentclass[12pt]{article}

\usepackage[a4paper,margin=1in]{geometry}
\usepackage{amsmath,amssymb,amsthm}
\usepackage[ruled,vlined]{algorithm2e}
\usepackage{mathtools}
\usepackage{stmaryrd}
\usepackage{graphicx}
\usepackage{enumitem}
\usepackage{hyperref}
\usepackage{microtype}
\usepackage{authblk}
\usepackage{titlesec}
\usepackage{anyfontsize}
\usepackage{float}
\usepackage{makecell}
\usepackage{multirow}
\usepackage{booktabs}
\usepackage{adjustbox}
\usepackage[table]{xcolor}
\usepackage{underscore}
\usepackage{caption}
\usepackage{siunitx}
\usepackage{tcolorbox}

\SetKwInOut{KwInit}{Initialize}

\newtheoremstyle{smallhead}
 {\topsep}
 {\topsep}
 {\itshape}
 {}
 {}
 {.}
 { }
 {\fontsize{11}{12}\bfseries\thmname{#1}\thmnumber{ #2}}

\theoremstyle{smallhead}
\newtheorem{theorem}{Theorem}
\newtheorem{assumption}{Assumption}
\newtheorem{remark}{Remark}
\newtheorem{lemma}{Lemma}
\newtheorem{claim}{Claim}
\newtheorem{proposition}{Proposition}

\titleformat{\section}
 {\normalfont\fontsize{11}{12}\selectfont\bfseries}
 {\thesection}
 {0.4em}
 {}

\titlespacing*{\section}
 {0pt}    % left
 {0.6em}   % before
 {0.3em}   % after

\titleformat{\subsection}
 {\normalfont\fontsize{11}{12}\selectfont\bfseries}
 {\thesubsection}
 {0.4em}
 {}

\titlespacing*{\subsection}
 {0pt}
 {0.4em}
 {0.2em}

\title{
Thermodynamic Regulation of Finite-Time Gibbs Training in Energy-Based Models\\[2em]
\large \textit{A Restricted Boltzmann Machine Study}
}

\author{
Görkem Can Süleymanoğlu\,
\href{https://orcid.org/0000-0003-4004-2267}
{\includegraphics[width=0.35cm]{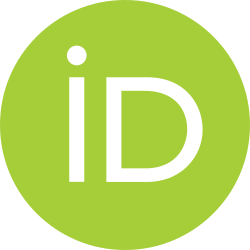}}
\thanks{
Owner and Software Developer, Kuanka Publishing LLC, Turkey.\\
Ph.D. Student in Economics, Selçuk University, Institute of Social Sciences, Turkey.\\
Email: \texttt{cangorkemsu@gmail.com}; \texttt{software@kuankallc.net}\\
Email (2): \texttt{254126001001@ogr.selcuk.edu.tr}\\
ORCID: \href{https://orcid.org/0000-0003-4004-2267}{0000-0003-4004-2267}\\
The code is publicly available on Zenodo 
(\href{https://doi.org/10.5281/zenodo.18778493}{10.5281/zenodo.18778493}) and GitHub (\href{https://github.com/cagasolu/sr-trbm}{github.com/cagasolu/sr-trbm}).
}
}

\date{}

\setlist[itemize]{
 topsep=0pt,
 partopsep=0pt,
 parsep=0pt,
 itemsep=0.8\baselineskip
}

\setlist[enumerate]{
 topsep=0pt,
 partopsep=0pt,
 parsep=0pt,
 itemsep=0.8\baselineskip
}

\begin{document}
\maketitle

\begin{abstract}
Restricted Boltzmann Machines (RBMs) are typically trained using finite-length Gibbs chains under a fixed sampling temperature. This practice implicitly assumes that the stochastic regime remains valid as the energy landscape evolves during learning. We argue that this assumption can become structurally fragile under finite-time training dynamics. This fragility arises because, in nonconvex energy-based models, fixed-temperature finite-time training can generate admissible trajectories with effective-field amplification and conductance collapse. As a result, the Gibbs sampler may asymptotically freeze, the negative phase may localize, and—without sufficiently strong regularization—parameters may exhibit deterministic linear drift. To address this instability, we introduce an endogenous thermodynamic regulation framework in which temperature evolves as a dynamical state variable coupled to measurable sampling statistics. Under standard local Lipschitz conditions and a two-time-scale separation regime, we establish global parameter boundedness under strictly positive $\ell_2$ regularization. We further prove local exponential stability of the thermodynamic subsystem and show that the regulated regime mitigates inverse-temperature blow-up and freezing-induced degeneracy within a forward-invariant neighborhood. Experiments on MNIST demonstrate that the proposed self-regulated RBM substantially improves normalization stability and effective sample size relative to fixed-temperature baselines, while preserving reconstruction performance. Overall, the results reinterpret RBM training as a controlled non-equilibrium dynamical process rather than a static equilibrium approximation.

\vspace{1em}
\end{abstract}

\vspace{1em}
\clearpage
\section{Introduction}

Energy-based models define probability distributions through parameterized energy mappings over high-dimensional configuration spaces. Among these, Boltzmann machines instantiate a canonical probabilistic formulation in which learning precisely consists of tuning parameters so that the induced Boltzmann distribution corresponds to the empirical data distribution \cite{ackley1985boltzmann}. Under the classical equilibrium view, parameter revisions are expressed in terms of expectations computed under the model’s stationary distribution.

Practical training procedures, on the other hand, depart substantially from this steady state idealization. Restricted Boltzmann Machines (RBMs) introduce architectural restrictions that allow efficient block Gibbs sampling \cite{smolensky1986harmony}, while Contrastive Divergence (CD) replaces equilibrium sampling with short-run Markov chains to approximate likelihood gradients \cite{hinton2002cd}. As a result, contemporary RBM learning operates in an intrinsically finite-time regime in which the sampler does not converge to the model’s stationary distribution. At this point, the equilibrium distribution is the theoretical target of learning, whereas training proceeds under finite-time stochastic dynamics.

Despite RBMs’ departure from the equilibrium idealization, the thermodynamic measures that govern random behavior—most notably the sampling temperature—are typically treated as static hyperparameters throughout training in these machines \cite{hinton2012practical}. This practice implicitly assumes that the stochastic regime of the Gibbs sampler persists as the energy surface evolves. Yet the landscape is continuously altered by learning. In particular, growth in the magnitude of weight rescales effective fields and amplifies energy differences between configurations. At fixed temperature, such rescaling alters transition probabilities without any compensatory adjustment of the stochastic regime. As weights grow, transitions may become increasingly rare, causing the Markov chain to mix slowly or nearly freeze. In other regimes, excessive stochasticity may obscure meaningful gradient data. Empirical studies have shown that learning dynamics in Boltzmann-type models are highly sensitive to temperature selection \cite{desjardins2010tempered}. Nevertheless, temperature is typically fixed independently of the model’s evolving internal state.

This discrepancy exposes a structural tension between equilibrium-based theory and out-of-equilibrium training practice. Classical Boltzmann learning, including RBMs, presumes access to a stationary distribution, whereas learning practice proceeds through finite-time probabilistic processes whose dynamical regimes are neither clearly characterized nor regulated. Consequently, the sampler may traverse qualitatively distinct stochastic phases during training, affecting mixing behavior, gradient reliability, and ultimately representational quality. The absence of an operational notion of stochastic stability leaves this regime drift largely unaddressed.

In this work, we introduce a formal distinction between \emph{classical thermal equilibrium} and \emph{dynamic, operational thermal non-equilibrium}. Classical thermal equilibrium refers to the stationary Boltzmann distribution assumed in the first learning derivation \cite{ackley1985boltzmann}. Dynamic, operational thermal non-equilibrium, by contrast, denotes a controlled stochastic regime in which sampling activity is explicitly measured and adaptively regulated across training epochs, without requiring the system to approach a stationary distribution. This perspective motivates treating temperature not as a fixed hyperparameter but as a dynamical control variable coupled to measurable statistics of the sampler’s stochastic activity.

We operationalize this perspective by defining an epoch-level flip-rate statistic ($r_{\text{flip}}$, denoted $r$ in the theorems and proofs) that quantifies sampling activity and embedding it within a feedback-based temperature update rule. In this context, the construction results in a closed-loop dynamical system that links parameter evolution, stochastic transitions, and thermodynamic control. We prove that, at fixed temperature, growth in effective fields leads to asymptotic freezing of the Gibbs sampler, and we confirm the local stability of the proposed feedback-controlled regime under limited sensitivity assumptions. Through integrating thermodynamic regulation directly into the learning dynamics, we obtain a control-theoretic stabilization method for finite-time Gibbs sampling.

Despite extensive research on deeper architectures and alternative training procedures, the direct regulation of sampling dynamics has largely been omitted from the learning rule itself. Much of the current work has focused on proving the practical effectiveness of RBM-based approaches or on the bias inherent in methods such as Contrastive Divergence; however, parameters defining internal processes, such as sampling temperature and Gibbs step count, have been treated as static hyperparameters held constant throughout the training period (\cite{hinton2006dbn, salakhutdinov2009deep}). Building on this distinction, the following section proposes two complementary hypotheses concerning the stability of learning under probabilistic approximation.

Although the analysis is developed in the Restricted Boltzmann Machine setting, the underlying instability mechanism does not rely on RBM-specific architectural details. The freezing phenomenon develops from the thermodynamic scaling in finite-time Markov chain sampling and therefore extends to energy-based models trained via short-run MCMC approximations. Consequently, the proposed feedback framework should be interpreted as a control-theoretic stabilization principle for finite-time stochastic approximation in energy-based learning, rather than as an RBM-specific heuristic.

\section{Learning Dynamics and Global Boundedness}

Let $\mathcal{A}$ denote a fixed architecture, $\mathcal{D}$ a fixed dataset, and $\mathcal{H}$ a fixed set of hyperparameters \textit{(the collection of default model settings)}. Let $\Theta \subset \mathbb{R}^d$ denote the parameter space. Consider two independent training runs differing only in random initialization with initial parameters $x,y \in \Theta$. Learning stability is defined as the invariance of the asymptotic parameter state. That is, the learning dynamics are said to be stable if different initializations yield structurally equivalent limiting parameter states \textit{($\mathcal{V}$ is a forward-invariant neighborhood.)},

\begin{equation}
\exists \mathcal{V} \subset \Theta \times \mathbb{R} \times [0,1]\ \text{ forward-invariant: }\ (\theta_0,\lambda_0,c_0) \in \mathcal{V} \Longrightarrow (\theta_t,\lambda_t,c_t) \in \mathcal{V},\ \quad \forall t \ge 0.
\end{equation}

In non-convex energy-based models, global invariance is generally not expected; accordingly, we focus on local structural stability within suitable parameter neighborhoods.

\begin{claim}[Absence of Structural Guarantee under Fixed Temperature]
\label{claim:fixed_temperature_no_global_stability}
Let $T > 0$ denote the fixed sampling temperature of the Gibbs transition kernel. Let $t \in \mathbb{N}$ denote the training epoch index (i.e., a full pass over the dataset). For fixed $\mathcal{A}$, $\mathcal{D}$, $\mathcal{H}$ and $T$: the training dynamics of RBM are determined by

\begin{equation}
\theta_{t+1} = \mathcal{T}_{\langle \mathrm{CD},T\rangle}(\theta_t).
\end{equation}

Under finite-time Contrastive Divergence training with fixed temperature, learning stability cannot, in general, be guaranteed. Assume:

\begin{enumerate}[label=\Alph*.]

\item $\mathcal{T}_{\langle \mathrm{CD},T\rangle}$ defines a stochastic approximation driven by finite-length Markov chains depending on $\mathcal{A}$ and $T$.

\item The induced update field depends on the parameter trajectory and remains initialization-sensitive through the empirical distribution $\mathcal{D}$.

\item The energy mapping $E_{\theta}(v,h)$ induced by $\mathcal{A}$ is non-convex in $\theta$ and admits multiple stationary regions under the fixed configuration $(\mathcal{A},\mathcal{D},\mathcal{H})$.

\end{enumerate}

Given Assumptions A--C, the fixed-temperature learning dynamics do not, in general, satisfy global asymptotic invariance. This instability, identified on the basis of the suppositions, motivates the introduction of an endogenous thermodynamic regulation mechanism. We now formalize this controlled regime and state the corresponding local stability result.
\end{claim}

In the fixed-temperature regime, the thermodynamic scale lacks an explicit dynamical representation depending on $T$. The absence of such a state equation prevents a closed-form stability analysis and precludes the construction of invariant-set guarantees.

\begin{claim}[Local Stability upon Endogenous Thermodynamic Regulation]
\label{claim:local_stability_upon_regulation}

In contrast to the fixed-temperature setting of Claim~\ref{claim:fixed_temperature_no_global_stability}, the learning dynamics are extended by treating the sampling temperature as an endogenous state variable. The system is thereby lifted to the augmented state space $\Theta \times \mathbb{R} \times [0,1]$, where the stochastic scale co-evolves jointly with the parameter trajectory. The sampling temperature is parameterized as

\begin{equation}
T_t = e^{\lambda_t},
\end{equation}

where ${\left(\lambda_t\right)}_{t \ge 0} \subset \mathbb{R}$ is a discrete-time thermodynamic state sequence defined recursively by

\begin{equation}
\lambda_{t+1} = \boldsymbol{\phi} \lambda_t - \eta_\lambda (r_t - c_t),
\end{equation}

with $0 < \boldsymbol{\phi} \le 1$ and $\eta_\lambda > 0$, where $\eta > 0$ denotes the learning rate of the parameter update. Here $r_t \in [0,1]$ denotes the observed flip-rate statistic of the Gibbs sampler at epoch $t$. The reference activity level $c_t$ evolves through exponential smoothing,

\begin{equation}
c_{t+1} = (1-\alpha)c_t + \alpha r_t,
\qquad 0 < \alpha \ll 1.
\end{equation}

The resulting coupled discrete-time dynamical system is

\begin{equation}
\theta_{t+1} = \mathcal{T}_{\mathrm{CD},\, e^{\lambda_t}}(\theta_t),
\qquad
\lambda_{t+1} = \boldsymbol{\phi} \lambda_t - \eta_\lambda (r_t - c_t),
\qquad
c_{t+1} = (1-\alpha)c_t + \alpha r_t.
\end{equation}

Assume local Lipschitz property of the CD operator in $(\theta,\lambda)$ and of the flip-rate function, together with a time-scale separation regime $\eta \ll \eta_\lambda \ll 1$. Under these conditions, the thermodynamic subsystem $(\lambda_t,c_t)$ converges locally exponentially to its operating point for frozen parameters, by standard discrete-time nonlinear asymptotic stability theory and two-time-scale stochastic approximation arguments (see, e.g., \cite{khalil2002nonlinear,vidyasagar2002nonlinear,borkar2008stochastic,karmakar2018twotimescale,dalal2018finite}). In this direction, there exists a forward-invariant neighborhood $\mathcal{V}$ in the augmented state space such that, for initial conditions in $\mathcal{V}$ with identical thermodynamic states, the coupled trajectories are locally bounded and evolve within a common regulated stochastic regime. In particular, there is a constant $C>0$ such that

\begin{equation}
\|\theta_t^{(a)} - \theta_t^{(b)}\| \le C |\theta_0^{(a)} - \theta_0^{(b)}|,
\quad
\text{for all } t \text{ such that } (\theta_t,\lambda_t,c_t)\in \mathcal{V}.
\end{equation}

Thus, the endogenous thermodynamic regulation prevents uncontrolled divergence induced by stochastic degeneracy and is associated with locally bounded learning dynamics, without asserting global contraction or asymptotic invariance.
\end{claim}

\begin{enumerate}[label=\Alph*., resume]
\setcounter{enumi}{3}
\item {\small \textbf{Classical Thermal Equilibrium (Stationary Boltzmann Distribution Assumption in Energy-Based Models)}}

In the classic Boltzmann Machine formulation, thermal equilibrium refers to the stationary relationship established between the probabilities of the global network states and the energy levels (\cite{ackley1985boltzmann, geman1984stochastic}). This assumption provides a mathematically convenient correspondence among stochastic patterns and parameter alterations; however, it is generally not possible to consistently satisfy this condition during practical RBM training (\cite{hinton2012practical, desjardins2010tempered}).

\item {\small \textbf{Dynamic, Operational Thermal Non-equilibrium (Adaptive Temperature Control Across Training Epochs)}}

Operational thermal non-equilibrium denotes a dynamically regulated stochastic regime in which the sampling process maintains a stable balance between exploration and convergence across successive training epochs without requiring convergence to a stationary Boltzmann distribution. In this setting, the sampling temperature cannot be treated as a fixed hyperparameter; instead, it is adaptively updated at the end of each epoch based on observed sampling statistics and used to initialize the next epoch. The inter-epoch propagation of temperature generates a feedback loop between the evolving energy landscape and the sampler's stochastic dynamics. Under this interpretation, the system operates in a controlled non-equilibrium regime that stabilizes sampling and supports reliable, informative gradient updates.

\end{enumerate}

This distinction directly motivates the perspective adopted in our study. Instead of treating thermal equilibrium as a static precondition, we reinterpret it as an active operational state that can occur and be monitored locally over time. By explicitly monitoring intrinsic indicators of the Gibbs sampler's behavior and conditioning learning updates on these indicators, we aim to bridge the gap between the conceptual foundations of Boltzmann machines and the realities of RBM training. This approach seeks to re-functionalize a fundamental principle of energy-based learning: meaningful parameter updates are most reliable when the stochastic dynamics remain within a statistically stable regime \cite{ackley1985boltzmann}.

\section{Stability of the Controlled Non-Equilibrium Regime}

The stability properties of the proposed temperature-control mechanism are now examined. The discussion proceeds in two steps. First, we clarify why keeping the temperature fixed (e.g., $T=1$ or otherwise) typically leads to degeneracy once the energy scale grows. We then prove that the feedback-controlled regime is locally stable under lenient conditions.

\subsection{Epoch-Level Dynamics}

{\small \textbf{Definition 1 (Epoch)}}
An epoch indexed by $t \in \mathbb{N}$ represents one complete training pass during which model parameters and thermodynamic control variables stay fixed while a Gibbs chain of length $K \in \mathbb{N}$ is executed. The index $t$ denotes the slow macroscopic time scale associated with temperature regulation and learning updates. Within each epoch, sampling proceeds on a fast microscopic time scale indexed by $k = 0,1,\dots,K$, where $k$ counts alternating visible–hidden Gibbs updates performed at the fixed temperature $T_t > 0$. The visible and hidden states at Gibbs step $k$ during epoch $t$ are denoted by $(v_t^{(k)}, h_t^{(k)}) \in \{0,1\}^{n_v} \times \{0,1\}^{n_h}$. The empirical flip rate at epoch $t$ is defined as

\begin{equation}
r_t = \frac{1}{2K} \sum_{k=1}^{K} \left( \frac{\|v_t^{(k)} - v_t^{(k-1)}\|_0}{n_v} + \frac{\|h_t^{(k)} - h_t^{(k-1)}\|_0}{n_h} \right),
\end{equation}

where $\|\cdot\|_0$ counts coordinate-wise state changes. The scalar quantity $r_t \in [0,1]$ denotes the average fraction of visible and hidden units that change state between successive Gibbs steps within epoch $t$. It serves as a measurable indicator of stochastic activity: values near zero correspond to nearly frozen dynamics, whereas larger values indicate increased mixing and randomness. While the flip-rate statistic is not equivalent to spectral gap or mixing time, it provides a directly observable surrogate for stochastic transition activity. After the Gibbs chain of epoch $t$ is completed and $r_t$ is computed, the system transitions to epoch $t+1$. This transition is not an additional Gibbs step but a macroscopic update event in which thermodynamic control variables are adjusted using statistics aggregated over epoch $t$. Let $c_t \in [0,1]$ denote an adaptive reference level for the flip rate. The variable $c_t$ tracks the typical magnitude of $r_t$ via exponential smoothing and thereby defines the desired operating regime of stochastic activity. Its evolution is given by

\begin{equation}
c_{t+1} = (1-\alpha)c_t + \alpha r_t,
\end{equation}

where $\alpha \in (0,1)$ determines the averaging time scale and governs the responsiveness of the reference level to fluctuations in $r_t$. Define the deviation (control error): $U_t \longrightarrow u$,

\begin{equation}
u_t = r_t - c_t.
\end{equation}

The log-temperature state variable $\lambda_t \in \mathbb{R}$ evolves according to a first-order discrete-time feedback recursion that couples thermodynamic regulation to monitored stochastic activity, defined by the update rule

\begin{equation}
\lambda_{t+1} = \boldsymbol{\phi} \lambda_t - \eta_\lambda u_t,
\end{equation}

with persistence factor $\boldsymbol{\phi} \in (0,1]$ and adaptation gain $\eta_\lambda > 0$. The effective sampling temperature at epoch $t$ is

\begin{equation}
T_t = e^{\lambda_t},
\end{equation}

which ensures strict positivity of the thermodynamic parameter. Taken together, the macroscopic thermodynamic regulation mechanism can be written compactly as

\begin{equation}
\begin{pmatrix}
\lambda_{t+1} \\
c_{t+1}
\end{pmatrix}
=
\begin{pmatrix}
\boldsymbol{\phi} \lambda_t - \eta_\lambda (r_t - c_t) \\
(1-\alpha)c_t + \alpha r_t
\end{pmatrix}.
\end{equation}

An operating point $(\lambda^*, c^*)$ satisfies the fixed-point condition $r(\lambda^*) = c^*$, which characterizes a self-consistent stochastic regime in which observed sampling activity matches its adaptive reference level. Within each epoch, the Gibbs chain indexed by $k$ operates under fixed $(\lambda_t, c_t)$, while the transition $t \mapsto t+1$ updates these macroscopic control variables using the empirical stochastic behavior observed during epoch $t$ (These variables are not elements of $\mathcal{H}$). This separation explicitly distinguishes microscopic sampling dynamics from macroscopic thermodynamic regulation and forms the structural basis of the stability analysis that follows.

\subsection{Energy-Scale Growth Under Fixed Temperature}

We begin with the fixed-temperature regime, where the sampling temperature is kept constant at $T > 0$. Throughout this subsection, vectors are regarded as elements of the finite-dimensional inner-product space $(\mathbb{R}^{n_v},\langle\cdot,\cdot\rangle)$ endowed with the standard Euclidean inner product $\langle x,y\rangle = x^\top y$ and the induced norm $\|x\|_2 = \sqrt{\langle x,x\rangle}$. Here $\sigma(z) := \frac{1}{1+e^{-z}}$ denotes the logistic sigmoid function.

\begin{lemma}[Exponential bound for the logistic variance term]
For all $z \in \mathbb{R}$, one has $\sigma(z)(1-\sigma(z)) \le C e^{-|z|}$ for some constant $C \in (0,\tfrac14]$.
\end{lemma}

\begin{proof}
For $z\ge0$, $\sigma(z)(1-\sigma(z))=\frac{e^{-z}}{(1+e^{-z})^2}$ and by symmetry in $|z|$ we obtain $\sigma(z)(1-\sigma(z))=\frac{e^{-|z|}}{(1+e^{-|z|})^2}\le\frac14 e^{-|z|}$, hence the claim holds for any $C\in(0,\tfrac14]$.
\end{proof}

\begin{theorem}[Asymptotic freezing under diverging effective inverse temperature]
\label{thm:freezing}

Let a binary RBM be trained using finite-length Gibbs chains at fixed sampling temperature $T>0$, and let $r_t$ denote the empirical flip rate at epoch $t$. Define $\beta^{\mathrm{field}}_t := \min_{(v,i)} \frac{|x_{t,i}(v)|}{T}$. Suppose that there exists a sequence $\beta^{\mathrm{field}}_t \longrightarrow \infty$ such that, for all sufficiently large $t$ and for every configuration $(v,h)$ visited during epoch $t$ and every coordinate $i$, the effective inverse temperature satisfies $\frac{|x_{t,i}(v)|}{T} \ge \beta^{\mathrm{field}}_t$. Then

\begin{equation}
r_t \longrightarrow 0.
\end{equation}
\end{theorem}

\begin{proof}
From the bound $\sigma(x_{t,i}/T)(1-\sigma(x_{t,i}/T)) \le C e^{-|x_{t,i}|/T}$ together with $\frac{|x_{t,i}(v)|}{T}\ge \beta^{\mathrm{field}}_t \to \infty$, we obtain an exponential upper bound for each coordinate-wise flip probability; averaging over units and Gibbs steps preserves this decay rate, so that $r_t \le C e^{-\beta^{\mathrm{field}}_t} \to 0$.
\end{proof}

Theorem~\ref{thm:freezing} is conditional: $(\exists\,\{\theta_t\}_{t\ge0}\ \text{admissible with }\lim_{t\to\infty}\beta^{\mathrm{field}}_{t}=\infty)\Rightarrow(\lim_{t\to\infty} r_t=0)$, whereas a trajectory-independent structural stability guarantee under fixed temperature would require $(\forall\,\{\theta_t\}_{t\ge0}\ \text{admissible},\ \sup_{t\ge0}\beta^{\mathrm{field}}_{t}<\infty)$; hence the existence of a single admissible trajectory satisfying $\lim_{t\to\infty}\beta^{\mathrm{field}}_{t}=\infty$ suffices to rule out any global structural stability guarantee. The argument concerns structural strength rather than probabilistic typicality. The existence of a single admissible diverging trajectory suffices to invalidate any trajectory-independent global stability guarantee for the fixed-temperature regime.

The instability conclusion is structural and existence-based: fixed-temperature finite-time Gibbs training does not admit trajectory-independent global guarantees. This does not imply that divergence occurs along typical training trajectories.

\begin{remark}
The result reveals a conditional thermodynamic degeneracy: whenever effective fields $|v^\top W_{t,\cdot j}|$ diverge while the sampling temperature remains fixed, the effective inverse temperature $\beta^{\mathrm{field}}_{t,j}(v) = \frac{|x_{t,j}(v)|}{T_t}$ grows without bound. Accordingly, the Gibbs sampler enters an asymptotically frozen regime in which stochastic transitions vanish. Importantly, this degeneration does not require global divergence of all weights; divergence along a single effective field direction is sufficient. The collapse of stochastic activity thus manifests a structural thermodynamic scaling effect rather than an artifact of a specific optimization algorithm. 

The assumption of effective-field divergence does not assert that weight norms must diverge under all practical training trajectories. It proves that fixed-temperature finite-time Gibbs training lacks systemic robustness: if an admissible trajectory induces unbounded growth in any effective field direction, freezing necessarily follows.
\end{remark}

Let $x_{t,j}(v) := v^\top W_{t,\cdot j} + b_{h_j,t}$ denote the effective field acting on hidden unit $j$ at epoch $t$.

\begin{lemma}[Conductance Collapse under Diverging Effective Field]
\label{lem:conductance}
We suppose the alternating-block Gibbs updates, where all hidden units are updated conditionally on the visible layer, followed by a conditional update of all visible units. Any transition between the sets $S$ and $S^c$ requires that unit $j$ changes state during the corresponding layer update. Consider the Gibbs transition kernel $P_t$ of a binary RBM at fixed temperature $T>0$ on the finite state space $\Omega$, and write $P_t(x,A):=\sum_{y\in A}P_t(x,y)$ for $A\subset\Omega$. Let $\pi_t$ denote a stationary distribution satisfying $\pi_t P_t=\pi_t$. Suppose there exists a coordinate $j$ such that $\inf_{x\in S} |x_{t,j}(v)| \longrightarrow \infty, \quad \text{as } t \longrightarrow \infty$. Define $Q_t(S,S^c) := \sum_{x \in S} \pi_t(x) P_t(x,S^c)$ as the stationary edge flow from $S$ to $S^c$, and define the conductance $\Phi_t := \inf_{S\subset\Omega,\;0<\pi_t(S)\le\frac12} \frac{Q_t(S,S^c)}{\pi_t(S)}$; then $\Phi_t \longrightarrow 0$ as $t \longrightarrow \infty$.
\end{lemma}

\begin{proof}
Fix a coordinate $j$ and define $S=\{(v,h)\in\Omega:h_j=1\}$ and $S^c=\{(v,h)\in\Omega:h_j=0\}$; under alternating block Gibbs updates any transition between $S$ and $S^c$ requires unit $j$ to flip during the hidden-layer update, and for binary Gibbs sampling at temperature $T>0$ the conditional flip probability satisfies $\sigma\!\Big(\frac{x_{t,j}(v)}{T}\Big)\Big(1-\sigma\!\Big(\frac{x_{t,j}(v)}{T}\Big)\Big)\le C e^{-\lvert x_{t,j}(v)\rvert/T}$ for some universal constant $C>0$, hence for all $x\in S$ we have $P_t(x,S^c) \le C e^{-\lvert x_{t,j}(v)\rvert/T}$, and defining $\beta^{\mathrm{field}}_t := \inf_{x\in S}\frac{\lvert x_{t,j}(v)\rvert}{T}$ it follows that $Q_t(S,S^c)=\sum_{x\in S}\pi_t(x)P_t(x,S^c)\le C e^{-\beta^{\mathrm{field}}_t}\pi_t(S)$, which implies $\frac{Q_t(S,S^c)}{\pi_t(S)}\le C e^{-\beta^{\mathrm{field}}_t}\to0$, so $\Phi_t\to0$.
\end{proof}

\begin{remark}
Since the Gibbs kernel at fixed temperature $T>0$ defines a finite, irreducible, and reversible Markov chain satisfying detailed balance with respect to $\pi_t$, Cheeger’s inequality holds: $\frac{\Phi_t^2}{2} \le \gamma_t \le 2\Phi_t$, where $\gamma_t$ denotes the spectral gap. Therefore, $\Phi_t \longrightarrow 0$ implies $\gamma_t \longrightarrow 0$, and thus asymptotic mixture degeneracy is established.
\end{remark}

\subsection{Gradient Distortion under Sampling Activity Collapse}

\begin{theorem}[Localization of the CD-$K$ Negative Phase under Vanishing Flip Rate]
\label{theo:gradient_distortion}
Consider CD-$K$ training with fixed finite $K$. Let $r_t$ denote the empirical flip rate at epoch $t$. If $r_t \longrightarrow 0$, then the empirical distribution generated by the $K$-step Gibbs chain concentrates on a subset of configurations whose expected Hamming radius converges to zero. Consequently, the negative-phase expectation in the CD-$K$ gradient depends only on a vanishingly small neighborhood of its initialization state.
\end{theorem}

\begin{proof}
Let $N_t$ denote the total number of coordinate flips during the $K$-step Gibbs chain at epoch $t$. Since at most $2K(n_v+n_h)$ coordinate updates occur within an epoch and $r_t$ is the average fraction of flipped coordinates per update, we have $\mathbb{E}[N_t] \le 2K(n_v+n_h)\, r_t$. If $r_t \longrightarrow 0$, it follows that $\mathbb{E}[N_t] \longrightarrow 0$. By Markov's inequality, $P(N_t \ge 1) \le \mathbb{E}[N_t] \longrightarrow 0$, implying that $P(N_t = 0) \longrightarrow 1$. Additionally, with probability tending to one, the $K$-step Gibbs chain performs no coordinate flips and the final configuration coincides with its initialization state. Therefore, the negative-phase sample at epoch $t$ converges in likelihood to a degenerate distribution supported on the initialization state.
\end{proof}

\begin{remark}[Two-Time-Scale Interpretation]
The local thermodynamic stability assessment is conducted under a frozen-parameter approximation, i.e., the parameter vector is fixed at $\theta = \theta^*$ while the thermodynamic subsystem $(\lambda_t, c_t)$ evolves. This corresponds to a standard two-time-scale separation in stochastic approximation theory (see, e.g., \cite{borkar2008stochastic}). The parameter behavior controlled by Contrastive Divergence evolves on a slower learning time scale determined by the learning rate, while the thermodynamic control variables evolve on an intermediate regulatory time scale governed by $(\eta_\lambda, \alpha)$. For sufficiently small feedback gains, the thermodynamic subsystem converges locally before significant drift in $\theta_t$ occurs. The stability results verify the conditional local stability of the controlled stochastic regime.
\end{remark}

The degeneration mechanism described above is distinct from the classical finite-$K$ bias of Contrastive Divergence. It comes from the thermodynamic collapse of stochastic activity ($r_t \to 0$), which suppresses transitions and localizes the negative phase. In this regime, the persistent Gibbs chain concentrates on a fixed configuration, leading to convergence of the empirical model statistics $M_t$ along freezing subsequences.

\begin{theorem}[Freezing-Induced Linear Drift under Fixed Temperature]
\label{theo:freezing_drift}
Suppose CD-$K$ training is carried out at fixed temperature $T>0$ with learning rate $\eta>0$. Let $D=\mathbb{E}_{\mathrm{data}}[vh^\top]$ and $M_t=\mathbb{E}_{\mathrm{model},t}[vh^\top]$. Assume that $r_t\to0$ and that along any subsequence on which freezing occurs, the negative-phase statistics $M_t$ converge almost surely to a limiting matrix $C$ (which may depend on the persistent initialization). Moreover, suppose that $D \neq C$. Then $W_{t+1}-W_t \longrightarrow \eta(D-C)$ almost surely and $W_t = W_0 + t\eta(D-C) + o(t)$ almost surely; in particular $\|W_t\|\ge t\eta\|D-C\| - o(t)$ and hence $\|W_t\|\to\infty$ with at least linear growth.
\end{theorem}

\begin{proof}
The update recursion reads $W_{t+1}=W_t+\eta(D-M_t)$, so by summing from $0$ to $t-1$ we obtain $W_t = W_0 + \eta\sum_{k=0}^{t-1}(D-M_k) = W_0 + t\eta(D-C) - \eta\sum_{k=0}^{t-1}(M_k-C)$; since $M_t \longrightarrow C$ almost surely and the sequence is uniformly bounded, the Ces\`aro averages $\frac{1}{t}\sum_{k=0}^{t-1}(M_k-C)$ converge to zero almost surely, which yields $W_t = W_0 + t\eta(D-C) + o(t)$ almost surely and therefore linear divergence whenever $D\neq C$.
\end{proof}

This consequence may imply that the condition $D=C$ occurs with nonzero probability. Alternatively, as reflected in Theorem~\ref{thm:globalboundness}, global boundedness continues to be true if certain hyperparameters in $\mathcal{H}$ are sufficiently small. Together, Theorem~\ref{thm:freezing} and Theorem~\ref{theo:freezing_drift} imply the existence of diverging trajectories for fixed temperature in the absence of sufficiently strong $\ell_2$ regularization. The existence of such a trajectory precludes a general trajectory-independent structural stability guarantee for the fixed-temperature regime. That situation establishes Claim~\ref{claim:fixed_temperature_no_global_stability}. In the freezing regime, model statistics decouple from data statistics. At this point, the CD updates produce linear parameter drift. Such behavior is compatible with the fixed-temperature CD recursion.

\begin{remark}[From Weight Divergence to Field Divergence]
Effective fields in an RBM are linear functionals of the weight matrix. If $\|W_t\| \longrightarrow \infty$, then there is a unit direction $U$ such that $\langle W_t, U \rangle \longrightarrow \infty$. If the set of configurations visited during training contains patterns with nonzero projection onto this direction, then at least a single effective field component $x_{t,i}(v)$ diverges along the training trajectory. Unless the training dynamics are confined to a degenerate subspace orthogonal to the diverging direction, weight divergence induces effective-field divergence.
\end{remark}

In this context, the thermodynamic subsystem $(\lambda_t, c_t)$ evolves on a faster time scale than the parameter trajectory $\theta_t$. Over the convergence horizon of $(\lambda_t, c_t)$, the cumulative parameter drift satisfies $\|\theta_{t+k}-\theta_t\| = O(k\eta)$, uniformly in $t$ as $\eta \to 0$. For sufficiently small time-scale ratio $\eta/\eta_\lambda$, this drift is negligible during thermodynamic convergence. Accordingly, the local stability analysis of the $(\lambda,c)$ subsystem may be carried out under a frozen-parameter approximation $\theta=\theta^*$, in agreement with standard two-time-scale stochastic approximation and classical singular perturbation arguments. The temperature, therefore, acts as an endogenous dynamical state evolving on the fast regulatory scale, and is not a fixed hyperparameter in $\mathcal{H}$.

\subsection{Global Parameter Boundedness under $\ell_2$ Regularization}

Define $G_t$ as the stochastic Contrastive Divergence gradient estimate at epoch $t$, let $\psi$ denote the block-diagonal $\ell_2$-regularization operator with positive coefficients $\psi_W$ and $\psi_b$, and equip $\mathbb{R}^d$ with the standard Euclidean norm $\|\theta\|_2 = \sqrt{\theta^\top \theta}$ while denoting by $\|\cdot\|$ the operator-induced norm; we consecutively consider the $\ell_2$-regularized parameter recursion $\theta_{t+1} = \theta_t + \eta (G_t - \Lambda \theta_t)$, where $\theta := (\mathrm{vec}(W), b^{(v)}, b^{(h)}) \in \mathbb{R}^d$ collects all parameters. The $\ell_2$-regularization term corresponds to classical weight decay, which adds a quadratic penalty $\frac{\psi}{2}\|\theta\|_2^2$ to the objective function and induces the shrinkage term $\psi \theta$ in the gradient update (see, e.g., \cite{krogh1992simple,bishop2006prml,goodfellow2016deep}). The regularization coefficients $(\psi_W,\psi_b)$ are fixed hyperparameters and are therefore elements of the configuration set $\mathcal{H}$.

\begin{equation}
\Lambda =
\begin{pmatrix}
\psi_W I & 0 & 0 \\
0 & \psi_b I & 0 \\
0 & 0 & \psi_b I
\end{pmatrix},
\qquad
\psi_W > 0,\ \psi_b > 0.
\end{equation}

We equip $\mathbb{R}^d$ with the Euclidean norm $\|\cdot\|_2$ and denote by $\|\cdot\|$ the induced operator norm. We impose the following standing assumptions: $(1)$ Finite minibatch size. $(2)$ Visible units are binary. $(3)$ Hidden activations lie in $(0,1)$. $(4)$ $0 < \eta \psi_{\max} < 2$, where 

\begin{equation}
\psi_{\max} = \max(\psi_W,\psi_b).
\end{equation}

\begin{theorem}[Global Parameter Boundedness]
\label{thm:globalboundness}
Under assumptions 1–4, the parameter sequence ${(\theta_t)}_{t\ge0}$ generated by the recursion is globally bounded:

\begin{equation}
\sup_{t \ge 0} \|\theta_t\|_2 < \infty.
\end{equation}
\end{theorem}

\begin{proof}
Because visible units are binary and hidden activations are bounded in $(0,1)$ by the sigmoid nonlinearity, all minibatch correlation statistics are uniformly bounded componentwise. Hence, there exists a constant $C>0$ such that

\begin{equation}
\|G_t\|_2 \le C \quad \text{for all } t \ge 0.
\end{equation}

The recursion can be written as $\theta_{t+1} = (I - \eta \Lambda)\theta_t + \eta G_t$. Since $\Lambda$ is symmetric positive definite, $I - \eta\Lambda$ is symmetric with eigenvalues $1 - \eta \psi_i$, where $\lambda_i \in \{\psi_W,\psi_b\}$. Under the assumption $0 < \eta \psi_{\max} < 2$, we have

\begin{equation}
|1 - \eta \lambda_i| < 1 \quad \text{for all } \lambda_i.
\end{equation}

Define $\varrho := \|I - \eta\Lambda\|_{2 \longrightarrow 2} = \max_{\psi_i \in \{\psi_W,\psi_b\}}|1 - \eta \psi_i|$. Then $0 < \varrho < 1$. Taking norms gives $\|\theta_{t+1}\|_2 \le \varrho \|\theta_t\|_2 + \eta C$. Iterating the inequality yields $\|\theta_t\|_2 \le \varrho^t \|\theta_0\|_2 + \eta C \sum_{k=0}^{t-1} \varrho^k$. Since $0 < \varrho < 1$, the geometric sum satisfies

\begin{equation}
\sum_{k=0}^{t-1} \varrho^k \le \frac{1}{1-\varrho}.
\end{equation}

Therefore, $\|\theta_t\|_2 \le \varrho^t \|\theta_0\|_2 + \frac{\eta C}{1-\varrho}$. Taking the supremum over $t \ge 0$ gives

\begin{equation}
\sup_{t\ge0} \|\theta_t\|_2 \le \max\!\left( \|\theta_0\|_2, \frac{\eta C}{1-\varrho} \right) < \infty.
\end{equation}

Hence, the parameter trajectory is globally bounded.
\end{proof}

The boundedness result does not rely on whether the sampling temperature is fixed or adaptively updated. Since the sigmoid activation satisfies $0 < \sigma(\cdot) < 1$ for all finite arguments, the Contrastive Divergence gradient remains uniformly bounded under both fixed and time-varying temperature schedules. Finally, global parameter boundedness follows solely from the presence of strictly positive $\ell_2$ regularization coefficients and is independent of the thermodynamic control mechanism.

\begin{remark}[Role of $\ell_2$ Regularization]
The global boundedness result critically relies on the strict positivity of the regularization coefficients $\psi_W > 0$ and $\psi_b > 0$. If $\psi_W = 0$ or $\psi_b = 0$, then the contraction factor satisfies

\begin{equation}
\varrho = \| I - \eta \Lambda \| = 1,
\end{equation}

and the affine recursion no longer admits a uniform geometric decay. In such a case, boundedness is not guaranteed and linear parameter drift may occur, as shown in Theorem~\ref{theo:freezing_drift}.
\end{remark}

\begin{remark}[Global Boundedness Does Not Imply Contraction]
The above theorem confirms global boundedness of the parameter trajectory but does not imply global contraction of the learning operator. The inequality

\begin{equation}
\|\theta_{t+1} - \tilde{\theta}_{t+1}\|_2 \le \varrho \|\theta_t - \tilde{\theta}_t\|_2 + \eta \|G_t - \tilde{G}_t\|_2
\end{equation}

does not guarantee a contraction mapping unless the gradient operator is globally Lipschitz with a sufficiently small constant. Since the Contrastive Divergence field is state-dependent and nonconvex, such a global contraction property cannot generally be ensured.
\end{remark}

\begin{remark}[Boundedness vs. Stochastic Stability]
$\ell_2$ regularization ensures global norm boundedness of $\{\theta_t\}$. However, boundedness is a geometric constraint and, by itself, does not characterize the random behavior of the sampling dynamics. While parameter norms are finite, the Gibbs chain may still exhibit poor mixing or reduced transition activity under fixed temperature.
\end{remark}

\begin{remark}[Norm Boundedness vs.\ Thermodynamic Freezing]
Positive $\ell_2$ regularization guarantees that the parameter vector stays globally bounded, and therefore that the effective fields $x_{t,i}(v)$ stay uniformly bounded along the trajectory. Freezing, however, is not determined by the magnitude of the fields alone, but by the effective inverse temperature $\beta^{\mathrm{field}}_{t,i}(v) = \frac{|x_{t,i}(v)|}{T_t}$. Norm boundedness by itself does not exclude freezing unless the sampling temperature is kept bounded away from zero. In the fixed-temperature regime with $T>0$, $\ell_2$ regularization indirectly prevents freezing by keeping $\beta^{\mathrm{field}}_{t}$ bounded. If the temperature were allowed to approach zero, freezing could arise even when the parameters are bounded. The adaptive thermodynamic mechanism introduced below handles this scale interaction by regulating $T_t$ and thereby controlling $\beta^{\mathrm{field}}_{t}$ within the forward-invariant neighborhood.
\end{remark}

\subsection{Local Stability of the Feedback-Controlled Regime \textnormal{(Proof of Claim~\ref{claim:local_stability_upon_regulation})}}

In the stability analysis, we interpret the expected flip rate as $r = r(\theta,\lambda)$. For the purpose of local thermodynamic stability analysis, we freeze the parameter vector at $\theta^*$ and define the induced mapping $r(\lambda) := r(\theta^*,\lambda)$, replacing the empirical quantity $r_t$ with its mean-field counterpart. This corresponds to a standard time-scale separation approximation in two-time-scale stochastic dynamics.

\begin{assumption}[Local Lipschitz Regularity of the Flip-Rate Function]
There exists a forward-invariant neighborhood $\mathcal{V}$ of $\lambda^*$ and a constant $L>0$ such that $|r(\lambda_1) - r(\lambda_2)| \le L |\lambda_1 - \lambda_2|$ for all $\lambda_1,\lambda_2 \in \mathcal{V}$.
\end{assumption}

We now examine the adaptive temperature dynamics as a two-dimensional discrete-time system. Let $(\lambda^*, c^*)$ denote an operating point satisfying $r(\theta^*,\lambda^*) = c^*$. Define variations from the operating point by $x_t = \lambda_t - \lambda^*$ and $y_t = c_t - c^*$.

\begin{remark}
The inequality $\boldsymbol{\phi} + \eta_\lambda L < 1$ can be interpreted as a small-gain condition ensuring that temperature feedback does not amplify the local sensitivity of the stochastic dynamics.
\end{remark}

While the previous theorem constitutes a sufficient small-gain condition for local stability, the system admits an exact characterization via linearization. We therefore examine the Jacobian of the deterministic mean-field map underlying the adaptive dynamics. Recall that the macroscopic update can be written as $F(\lambda,c)=\big(\boldsymbol{\phi} \lambda - \eta_\lambda (r(\lambda)-c),\, (1-\alpha)c + \alpha r(\lambda)\big)$. Assume additionally that $r$ is differentiable at $\lambda^*$ and let $s := \partial_\lambda r(\theta^*,\lambda^*)$.

\begin{theorem}[Jacobian Stability Characterization]
The operating point $(\lambda^*,c^*)$ is locally exponentially stable if and only if the eigenvalues of the Jacobian matrix $J=\begin{pmatrix}\boldsymbol{\phi} - \eta_\lambda s & \eta_\lambda \\ \alpha s & 1-\alpha\end{pmatrix}$ lie strictly inside the unit disk.
\end{theorem}

\begin{proof}
Differentiating $F$ yields $\partial_\lambda F_1=\boldsymbol{\phi} -\eta_\lambda r'(\lambda)$, $\partial_c F_1=\eta_\lambda$, $\partial_\lambda F_2=\alpha r'(\lambda)$, and $\partial_c F_2=1-\alpha$, where $r'(\lambda)$ denotes $\partial_\lambda r(\theta^*,\lambda)$. Evaluating at $(\lambda^*,c^*)$ gives the matrix of the Jacobian with $s := r'(\lambda^*)$. By discrete-time linearization theory, local exponential stability of the nonlinear system is equivalent to asymptotic stability of the linearized model $z_{t+1}=Jz_t$ (see, e.g., \cite{vidyasagar2002nonlinear}). For a two-dimensional discrete-time linear system, Schur stability is equivalent to the Jury conditions $|\det J|<1$, $1+\operatorname{tr}(J)+\det(J)>0$, and $1-\operatorname{tr}(J)+\det(J)>0$. A straight computation yields $\operatorname{tr}(J)=\boldsymbol{\phi} -\eta_\lambda s+1-\alpha$ and $\det(J)=(\boldsymbol{\phi} -\eta_\lambda s)(1-\alpha)-\eta_\lambda\alpha s=\boldsymbol{\phi} (1-\alpha)-\eta_\lambda s$, which completes the proof.
\end{proof}

Under the differentiability assumption on $r$ and the spectral condition $\rho(J) < 1$, the fixed point $(\lambda^*,c^*)$ is locally exponentially stable for the mean-field thermodynamic dynamics. Consequently, there exists a forward-invariant neighborhood $\mathcal{V}$ of $(\lambda^*,c^*)$ such that trajectories initialized in $\mathcal{V}$ converge exponentially to the operating point. Standard results ensure existence of a local Lyapunov function.

\begin{remark}
The identity $\det(J)=\boldsymbol{\phi}(1-\alpha)-\eta_\lambda s$ makes the feedback structure transparent: $s=r'(\lambda^*)$ measures local stochastic sensitivity, $\eta_\lambda$ is the feedback gain, and $\boldsymbol{\phi} (1-\alpha)$ represents intrinsic damping. Instability occurs precisely when the amplification term $\eta_\lambda s$ exceeds this damping level.
\end{remark}

\begin{proposition}[Stochastic Boundedness]
Within the local mean-field approximation where the parameter vector is fixed at $\theta^*$, assume that the empirical flip rate can be written as $r_t = r(\theta^*,\lambda_t) + \xi_t$, where $\mathbb{E}[\xi_t \mid \lambda_t] = 0$ and $|\xi_t| \le \delta$ almost surely for some $\delta > 0$. If the mean-field Jacobian satisfies $\rho(J) < 1$, then the stochastic thermodynamic trajectory $(\lambda_t, c_t)$ is almost surely bounded.
\end{proposition}

\begin{proof}
Let $z_t = (\lambda_t - \lambda^*,\, c_t - c^*)^\top$ so that $z_{t+1}=Jz_t+\eta_t$ with bounded $\eta_t$; the condition $\rho(J)<1$ implies the existence of deterministic constants $M>0$ and $\beta\in(0,1)$ such that $\|J^t\|\le M\beta^t$, hence $\|z_t\|\le M\beta^t\|z_0\|+\sum_{k=0}^{t-1}M\beta^{t-1-k}\|\eta_k\|$, which is uniformly bounded, and therefore $\sup_{t \ge 0}\|(\lambda_t,c_t)\|<\infty$ almost surely.
\end{proof}

Given that the thermodynamic factors are locally bounded and the temperature stays bounded away from zero within $\mathcal{V}$, the effective inverse temperature cannot diverge, and the freezing mechanism characterized in Theorem~\ref{thm:freezing} cannot be triggered. In this direction, under local Lipschitz regularity of the CD operator together with time-scale separation, the parameter trajectory $\theta_t$ stays bounded within the forward-invariant neighborhood $\mathcal{V}$, thereby confirming Claim~\ref{claim:local_stability_upon_regulation} \textit{(a stabilization property of the regulated learning dynamics)}.

\subsection{Macroscopic Thermodynamic Stabilization}

In the code implementation, the quantity $\bar{\Delta F}_t$ in the temperature rule $T_t = e^{\lambda_t} + \kappa \bar{\Delta F}_t$ is computed as a running average of the epoch-level free-energy gap; for theoretical clarity, we analyze instead the cumulative energy gap $\bar{\mathcal R}_t$. It is asymptotically equivalent to $\bar{\Delta F}_t$ under bounded trajectories of temperature and parameters (see, e.g., \cite{hardy1949divergent}). Now, we introduce a macroscopic thermodynamic coherence rule $T_t = \kappa \bar{\mathcal R}_t$, where $\bar{\mathcal R}_t = \frac{1}{t} \sum_{k=1}^{t} \mathcal R_k$ and $\mathcal R_t = \left| \mathbb{E}_{\mathrm{data}}\!\left[E_{\theta_t}\right] - \mathbb{E}_{\mathrm{model},t}\!\left[E_{\theta_t}\right] \right|$, with $\kappa > 0$ controlling the strength of energy–temperature coupling; because $T_t$ evolves endogenously across epochs, it is not a fixed element of $\mathcal{H}$.

\begin{theorem}[Macroscopic Thermodynamic Coherence]
\label{thm:macroscopic_coherence}

Suppose:

\begin{enumerate}[label=\Alph*., resume]
\item The parameter trajectory $\{\theta_t\}$ belongs to a compact subset of $\Theta$.
\item The learning rate $\eta$ is sufficiently small.
\item The stochastic gradient noise has a bounded second moment.
\item The Contrastive Divergence updates satisfy stochastic approximation conditions.
\end{enumerate}

Then the following statements hold:

\begin{enumerate}[label=\Alph*., resume]
\item \textbf{Global boundedness of temperature.} 
$\sup_{t \ge 1} T_t < \infty$. Bounded parameters imply bounded energies, which imply bounded $\mathcal R_t$, bounded $\bar{\mathcal R}_t$, and therefore bounded $T_t$.

\item \textbf{Energy--temperature equivalence.} 
$T_t \longrightarrow 0$ if and only if $\bar{\mathcal R}_t \longrightarrow 0$. Since $T_t = \kappa \bar{\mathcal R}_t$ with $\kappa > 0$, vanishing temperature is equivalent to vanishing cumulative energy imbalance.

\item \textbf{Persistence under macroscopic energy imbalance.} 
If there exists $\varepsilon > 0$ such that $\liminf_{t \longrightarrow \infty} \mathcal R_t \ge \varepsilon$, then $\liminf_{t \longrightarrow \infty} \bar{\mathcal R}_t \ge \varepsilon$, and therefore $\liminf_{t \longrightarrow \infty} T_t \ge \kappa \varepsilon > 0$. Since $\bar{\mathcal R}_t$ is the Ces\`aro average of the nonnegative sequence $\{\mathcal R_t\}$, a positive lower bound on $\mathcal R_t$ implies the same lower bound on $\bar{\mathcal R}_t$.
\end{enumerate}
\end{theorem}

\begin{proof}
Positive $\ell_2$-regularization is assumed, together with fixed architecture $\mathcal{A}$, dataset $\mathcal{D}$, and hyperparameters $\mathcal{H}$. The temperature is not an element of $\mathcal{H}$; here, it evolves endogenously as a state variable of the learning dynamics. It is true that Theorem~\ref{thm:globalboundness} guarantees that $\sup_t \|\theta_t\|_2 < \infty$, independently of the temperature trajectory. The restricted Boltzmann machine energy is linear in $\theta$, so bounded parameters imply bounded energies. Accordingly, $\mathcal R_t$ and its Ces\`aro averages $\bar{\mathcal R}_t$ are bounded. Therefore, the temperature sequence $T_t = \kappa \bar{\mathcal R}_t$ is bounded, and statements (A)--(C) are precisely true.
\end{proof}

\begin{remark}[Conditional Nature]
The boundedness of the temperature sequence depends on the boundedness of the parameter trajectory $\{\theta_t\}$. The boundedness of $\{\theta_t\}$ is assumed here and is not derived from the fixed-temperature CD dynamics. The cumulative temperature rule guarantees thermodynamic coherence but, by itself, does not ensure boundedness of the parameter behavior. A fully autonomous global stability analysis of the coupled parameter–temperature system seems to be open.
\end{remark}

\begin{theorem}[Macroscopic Scaling Prevents Inverse-Temperature Blow-Up]
\label{thm:macro_drift_control}
Let the temperature evolve according to $T_t = \kappa \bar{\mathcal R}_t$, where $\bar{\mathcal R}_t = \frac{1}{t}\sum_{k=1}^{t} \mathcal R_k$ and $\kappa>0$, and define the effective inverse temperature $\beta^{\mathrm{norm}}_t = \frac{\|\theta_t\|_2}{T_t}$. Assume that whenever $\|\theta_t\|_2 \to \infty$, there exists $\varepsilon>0$ such that $\liminf_{t\to\infty}\mathcal R_t \ge \varepsilon$. Then the sequence $(\beta^{\mathrm{norm}}_t)_{t\ge 0}$ is uniformly bounded, i.e., $\sup_t \beta^{\mathrm{norm}}_{t} < \infty$.
\end{theorem}

\begin{proof}
Suppose, by contradiction, that there exists a subsequence $t_j$ such that $\|\theta_{t_j}\|_2 \to \infty$ and $\beta^{\mathrm{norm}}_{t_j} \to \infty$. Since $T_{t_j}=\kappa \bar{\mathcal R}_{t_j}$, we have $\beta^{\mathrm{norm}}_{t_j}=\|\theta_{t_j}\|_2/(\kappa \bar{\mathcal R}_{t_j})$, so divergence requires $\bar{\mathcal R}_{t_j}\to 0$. However, the assumption $\liminf_{t\to\infty}\mathcal R_t \ge \varepsilon>0$ implies $\liminf_{t\to\infty}\bar{\mathcal R}_t \ge \varepsilon$ (since $\bar{\mathcal R}_t$ is the Ces\`aro average of a nonnegative sequence), hence $T_t$ cannot converge to zero, yielding a contradiction. Therefore $\sup_t \beta^{\mathrm{norm}}_t<\infty$.
\end{proof}

Such a condition is natural in linear energy parameterizations, except in the trivial case where model and data statistics match exactly along the diverging parameter direction. In that degenerate case, the gradient would vanish, and divergence would not persist.

\begin{remark}
The divergence mechanism in Theorem~\ref{theo:freezing_drift} operates under the freezing regime established in Theorem~\ref{thm:freezing}, which requires $\beta^{\mathrm{field}}_{t} \longrightarrow \infty$. Since macroscopic thermodynamic scaling keeps $\beta^{\mathrm{norm}}_{t}$ bounded, the inverse-temperature blow-up required for freezing-induced linear drift cannot arise.
\end{remark}

\section{Model and Equilibrium-Guided Training}

RBMs are part of the energy-based model family, which draws on principles of statistical mechanics. In this framework, the model assigns lower energy values to data configurations, making them more probable under a Gibbs distribution. Learning, then, is the process of shaping an energy landscape in which the data points reside in stable, low-energy valleys \cite{ising1925beitrag, metropolis1953equation, ackley1985boltzmann, smolensky1986harmony, lecun2006tutorial}. Now, we consider a binary Restricted Boltzmann Machine (RBM) consisting of a visible layer $v \in {0,1}^{n_v}$ and a hidden layer $h \in {0,1}^{n_h}$, with symmetric weights $W \in \mathbb{R}^{n_v \times n_h}$ and bias vectors $b_v$ and $b_h$. 

\begin{equation}
E(v,h) = - v^\top b_v - h^\top b_h - v^\top W h
\end{equation}

Marginalizing over hidden states yields the free energy of a visible configuration,

\begin{equation}
F_{T}(v) = - v^\top b_v - T \sum_{j=1}^{n_h} \log\left( 1 + e^{\frac{v^\top W_j + b_{h_j}}{T}}\right),
\end{equation}

which determines the model distribution over visible states \cite{ackley1985boltzmann, hinton2012practical}.

\subsection{Temperature-Controlled Gibbs Sampling}

Sampling in RBMs is performed using alternating Gibbs updates, corresponding to a Markov chain Monte Carlo procedure that targets the model’s Gibbs distribution \cite{metropolis1953equation, geman1984stochastic}. In finite-time training, however, the sampler's stochastic regime critically depends on the effective temperature at which transitions happen. To regulate this regime in a principled manner, we introduce an adaptive temperature parameter defined at epoch $t$ by

\begin{equation}
T_t = e^{\lambda_t} + \kappa \left|\mathbb{E}_{\mathrm{data}}\!\left[F_{\theta_t}(v)\right] - \mathbb{E}_{\mathrm{model},t}\!\left[F_{\theta_t}(v)\right]\right|, \qquad \kappa>0,
\end{equation}

where $\lambda_t \in \mathbb{R}$ is an unconstrained internal thermodynamic state variable and the second term represents a macroscopic correction proportional to the instantaneous free-energy imbalance between the data \& model distributions. The exponential component guarantees strict positivity of the temperature, ensuring $T_t > 0$ for all $t$, while permitting unconstrained evolution of the internal state variable. The additive macroscopic term adjusts the stochastic scale whenever a global discrepancy between actual and simulated energies persists. Under temperature scaling, the conditional activation probabilities of the Gibbs sampler become

\begin{equation}
p(h_j = 1 \mid v) = \frac{1}{1 + e^{-\frac{v^\top W_j + b_{h_j}}{T_t}}}, \qquad p(v_i = 1 \mid h) = \frac{1}{1 + e^{-\frac{h^\top W_i + b_{v_i}}{T_t}}}.
\end{equation}

Thus, the effective temperature directly rescales the magnitude of energy differences governing stochastic transitions. Lower temperatures sharpen the conditional distributions and promote nearly deterministic updates, whereas higher temperatures increase entropy and enhance exploratory mixing. The joint distribution of the restricted Boltzmann machine under temperature scaling takes the canonical Boltzmann–Gibbs form

\begin{equation}
p_{T_t}(v,h) = \frac{1}{Z(T_t)} e^{-\frac{E(v,h)}{T_t}},
\end{equation}

where $Z(T_t)$ denotes the partition function at temperature $T_t$. This Gibbs distribution admits a variational characterization: it is the unique minimizer of the Helmholtz free energy mapping

\begin{equation}
\mathcal{F}_{T_t}(p) = \mathbb{E}_p[E(v,h)] - T_t \mathcal{S}(p),
\end{equation}

where $\mathcal{S}(p) = -\sum_{v,h} p(v,h)\log p(v,h)$ is the Shannon entropy. The temperature, therefore, acts as a Lagrange multiplier regulating the trade-off between expected energy and entropy.

\subsection{Flip Dynamics along with Energy Stability}

We monitor flip statistics at the visible layer as an operational indicator of stochastic activity.
In bipartite RBMs, visible and hidden layers are coupled through alternating conditional Gibbs updates \cite{geman1984stochastic}; therefore, suppression of visible transitions suffices to detect freezing of the sampling dynamics. Hidden-layer flips can be incorporated analogously, but do not modify the stability argument and are omitted for simplicity. In the implementation based on Persistent Contrastive Divergence (PCD-$K$), the flip-rate statistic is computed from the persistent negative-phase chain. Let $v_t^{\mathrm{prev}}$ denote the visible persistent-chain state before the Gibbs updates of a minibatch during epoch $t$, and $v_t^{\mathrm{neg}}$ the state after the $K$ Gibbs steps. The minibatch flip rate is defined as

\begin{equation}
r_t^{(b)} = \frac{1}{n_v} \left\| v_t^{\mathrm{neg}} - v_t^{\mathrm{prev}}\right\|_0,
\end{equation}

and the epoch-level flip statistic is obtained through averaging over minibatches,

\begin{equation}
r_t = \mathbb{E}_{\mathrm{minibatch}} \left[r_t^{(b)}\right].
\end{equation}

Although $r_t$ is not directly related to the spectral gap or mixing time \cite{roberts1997weak}, vanishing flip rate implies suppression of coordinate-wise transition chances. As established in Theorem~\ref{thm:freezing} and Lemma~\ref{lem:conductance}, this dampening leads to conductance collapse and mixing degeneration in the sense of classical Markov chain theory \cite{metropolis1953equation}. The flip rate thus functions as a practical, effective surrogate for detecting incipient freezing in finite-time Gibbs training. To complement this microscopic diagnostic, we compute the mean absolute energy variation along the chain,

\begin{equation}
\overline{|\Delta E|} = \frac{1}{K-1} \sum_{k=2}^{K} \big| E^{(k)} - E^{(k-1)} \big|.
\end{equation}

At steady-state statistical mechanics, bounded energy fluctuations characterize thermodynamically regulated regimes, whereas irregular or explosive variations signal non-equilibrium drift. Monitoring $\overline{|\Delta E|}$ therefore captures energetic instability that may not immediately appear as coordinate freezing.

\subsection{Hybrid Thermodynamic Stabilization and Structural Closure}

Stabilization of local stochastic activity alone does not explicitly constrain cumulative global energy imbalance. Persistent differences between the expected data-free energy and the model-free energy may gradually rescale the energy surface, even if short-run stochastic transitions are regulated. To address this long-horizon effect, we introduce a macroscopic correction term based on the Ces\`aro-averaged free-energy gap, yielding the hybrid temperature rule,

\begin{equation}
T_t = e^{\lambda_t} + \kappa \,\overline{\Delta F}_t,
\qquad \kappa > 0.
\end{equation}

The hybrid rule jointly stabilizes microscopic sampling activity and macroscopic energy balance. In this formulation, the purely macroscopic scaling rule $T_t=\kappa \bar{\mathcal R}_t$ appears as a limiting case obtained when the microscopic feedback component $e^{\lambda_t}$ is suppressed. Accordingly, the implemented hybrid dynamics unify short-run stochastic regulation with long-run thermodynamic scaling within a single temperature evolution,

\begin{equation}
\beta^{\mathrm{field}}_{t,j}(v) = \frac{|x_{t,j}(v)|}{T_t},
\end{equation}

thereby blocking the freezing pathway identified in Theorems~\ref{thm:freezing}--\ref{theo:freezing_drift}. Importantly, the learning objective remains unchanged. Parameter revisions continue to follow the classical contrastive gradient structure of energy-based learning. We therefore obtain the following structural conclusion:

\begin{proposition}[Structural Stabilization under Hybrid Thermodynamic Control]
Assume the local Jacobian stability condition $\rho(J)<1$, bounded stochastic gradient noise, and sufficiently small learning rate $\eta$. There is a forward-invariant neighborhood in the augmented state space within which the hybrid temperature dynamics prevent freezing, conductance collapse, and linear parameter divergence induced by negative-phase localization.
\end{proposition}

The analysis focuses on local regime stabilization and does not address global convergence. Instead, the proofs show that regulating internal temperature blocks the structural path that would lead toward thermodynamic degeneracy. This holds within the forward-invariant regime established by the controlled dynamics.

Next, we briefly review the relevant literature and then present the experimental results.

\section{Related Works}

Energy-based learning emerged from statistical mechanics and Markov chain sampling theory, in which learning is traditionally formulated under equilibrium deductions. Within this scheme, Boltzmann learning is based on the principle of equating the equilibrium statistics of the data and model distributions. Restricted Boltzmann Machines, however, are trained using finite-time Gibbs dynamics together with approximate gradient procedures such as Contrastive Divergence, which do not require the system to reach equilibrium at each update. The literature, therefore, reflects both the thermodynamic foundations of Boltzmann learning and the subsequent algorithmic evolution of RBM training methods. Table~\ref{tab:ebm_timeline} traces these conceptual and methodological developments and emphasizes the assumptions underlying the classical fixed-temperature RBM framework.

\begin{table}[H]
\centering
\caption{Conceptual lineage of energy-based models and sampling methods underlying RBMs.}
\label{tab:ebm_timeline}

\renewcommand{\arraystretch}{2.2}

\resizebox{\textwidth}{!}{%
\begin{tabular}{|l|l|l|}
\hline\
\textbf{Work} &
\textbf{Core contribution} &
\textbf{Relation to the classical RBM} \\

\shortstack[l]{Ising (1925)\\\cite{ising1925beitrag}} &
\shortstack[l]{Binary interacting variables described by an energy\\ function under thermal equilibrium.} &
\shortstack[l]{Provides the canonical mathematical form of\\ pairwise energy interactions adopted by RBMs.} \\

\shortstack[l]{Metropolis et al. (1953)\\\cite{metropolis1953equation}} &
\shortstack[l]{Markov Chain Monte Carlo sampling\\ from Boltzmann distributions.} &
\shortstack[l]{Defines the sampling principle underlying\\ Gibbs updates in RBM training.} \\

\shortstack[l]{Little (1974)\\\cite{little1974existence}} &
\shortstack[l]{Persistent collective neural states\\ emerge from network dynamics.} &
\shortstack[l]{Early evidence of stable collective configurations\\ analogous to low-energy RBM states.} \\

\shortstack[l]{Hopfield (1982)\\\cite{hopfield1982neural}} &
\shortstack[l]{Neural computation formulated as\\ energy minimization with attractors.} &
\shortstack[l]{Introduces explicit energy landscapes that\\ motivate stochastic generalizations such as RBMs.} \\

\shortstack[l]{Kirkpatrick et al. (1983)\\\cite{kirkpatrick1983optimization}} &
\shortstack[l]{Optimization via simulated annealing.} &
\shortstack[l]{Demonstrates the operational role of temperature\\ in navigating complex energy landscapes.} \\

\shortstack[l]{Geman \& Geman (1984)\\\cite{geman1984stochastic}} &
\shortstack[l]{Gibbs distributions and stochastic\\ relaxation for inference.} &
\shortstack[l]{Forms the probabilistic foundation of\\ Gibbs sampling used in RBMs.} \\

\shortstack[l]{Ackley et al. (1985)\\\cite{ackley1985boltzmann}} &
\shortstack[l]{Learning by matching equilibrium statistics\\ of data and model distributions.} &
\shortstack[l]{Establishes the theoretical learning objective\\ approximated by the classical fixed-temperature\\ RBM baseline.} \\

\shortstack[l]{Smolensky (1986)\\\cite{smolensky1986harmony}} &
\shortstack[l]{Harmony theory: learning as maximization\\ of a global energy-like objective.} &
\shortstack[l]{Conceptual precursor to Boltzmann Machines\\ and Restricted Boltzmann Machines.} \\

\shortstack[l]{Neal (1996)\\\cite{neal1996sampling}} &
\shortstack[l]{Tempered transitions for sampling\\ multimodal distributions.} &
\shortstack[l]{Points to limitations of fixed-temperature\\ MCMC when sampling multimodal energy terrains.} \\

\shortstack[l]{Roberts et al. (1997)\\\cite{roberts1997weak}} &
\shortstack[l]{Optimal scaling and convergence\\ of Metropolis algorithms.} &
\shortstack[l]{Builds theoretical insight into mixing behavior\\ relevant to RBM Gibbs sampling.} \\

\shortstack[l]{Jarzynski (1997)\\\cite{jarzynski1997nonequilibrium}} &
\shortstack[l]{Exact free-energy relations\\ far from equilibrium.} &
\shortstack[l]{Conceptually motivates consideration of nonequilibrium\\ effects in thermodynamic learning systems.} \\

\shortstack[l]{Crooks (1999)\\\cite{crooks1999entropy}} &
\shortstack[l]{Entropy production and nonequilibrium\\ fluctuation theorems.} &
\shortstack[l]{Provides theoretical insight into entropy production\\ within nonequilibrium statistical systems.} \\

\shortstack[l]{Hinton (2002)\\\cite{hinton2002cd}} &
\shortstack[l]{Contrastive Divergence for approximate\\ equilibrium gradient estimation.} &
\shortstack[l]{Defines the exact training algorithm used\\ in the classical fixed-temperature RBM baseline.} \\

\shortstack[l]{LeCun et al. (2006)\\\cite{lecun2006tutorial}} &
\shortstack[l]{Unified energy-based learning framework.} &
\shortstack[l]{Places RBMs within a broader class of\\ energy-based models.} \\

\shortstack[l]{Hinton et al. (2006)\\\cite{hinton2006dbn}} &
\shortstack[l]{Greedy layer-wise training of deep\\ belief networks.} &
\shortstack[l]{Extends RBMs as modular building blocks\\ for deep generative models.} \\

\shortstack[l]{Salakhutdinov \& Hinton (2009)\\\cite{salakhutdinov2009deep}} &
\shortstack[l]{Deep Boltzmann Machines with\\ multilayer interactions.} &
\shortstack[l]{Generalizes RBMs, amplifying equilibrium\\ and sampling challenges.} \\

\shortstack[l]{Desjardins et al. (2010)\\\cite{desjardins2010tempered}} &
\shortstack[l]{Parallel tempering for improved\\ RBM sampling.} &
\shortstack[l]{Demonstrates sensitivity of RBM behavior\\ to temperature control.} \\

\shortstack[l]{Hinton (2012)\\\cite{hinton2012practical}} &
\shortstack[l]{Practical training heuristics\\ for RBMs.} &
\shortstack[l]{Defines the standard experimental\\ configuration used as the classical baseline.} \\
\hline

\end{tabular}}
\end{table}

Though classical Boltzmann learning assumes access to equilibrium distributions, practical RBM training relies on finite-time Gibbs chains and approximate gradient procedures such as Contrastive Divergence. This distinction between equilibrium theory and operational sampling dynamics has motivated research on mixing behavior, tempering strategies, and convergence analysis. Table~\ref{tab:ebm_timeline} outlines the key contributions that formed the classical fixed-temperature RBM baseline and clarifies the thermodynamic assumptions on which it rests.

\section{Code-Level Experimental Protocol}

All experiments were conducted under Ubuntu Linux in a controlled CUDA-enabled GPU environment. Each training run was executed as a single isolated process on a dedicated GPU. This configuration ensured exclusive resource assignment and prevented cross-run interference. Random seeds were fixed at both CPU and GPU levels to guarantee full reproducibility. The implementation is provided as a standalone Python script (\texttt{srtrbm\_project0.py}), comprising model definition, thermodynamic control dynamics, Persistent Contrastive Divergence training, Annealed Importance Sampling (AIS), MCMC diagnostics, and automated statistical comparison routines. Temperature strategies are controlled via a single configuration argument (\texttt{fixed\_temperature}). This design confirms that all scientific comparisons isolate the temperature-regulation mechanism while maintaining identical architectural and optimization settings. The algorithmic components are device-agnostic at the tensor abstraction level. Although experiments were conducted on CUDA-enabled GPUs, adapting to alternative accelerator backends (e.g., XLA-based TPU systems) would primarily require modifications to the device abstraction layer rather than structural changes to the learning algorithm.

\subsection{SR-TRBM Algorithm}

\begin{algorithm}[H]
\caption{Self-Regulated Thermodynamic RBM (SR-TRBM)}
\label{alg:srtrbm}
\small

\KwIn{
Dataset $\mathcal{D}$,
epochs $E$,
Gibbs steps $K$,
learning rate $\eta$,
feedback gain $\eta_\lambda$,
flip smoothing $\alpha$,
macro scale $\kappa$
}

\KwInit{
Initialize $W \sim \mathcal{N}(0,\sigma^2)$\;
Initialize $b_v \leftarrow 0$, $b_h \leftarrow 0$\;
Initialize $\lambda_0 \leftarrow 0$, $c_0 \leftarrow 0$\;
Initialize $\bar{\Delta F}_0 \leftarrow 0$\;
Initialize persistent chain $v^{(0)} \sim \mathrm{Bernoulli}(0.5)$
}

\For{$t = 1$ \KwTo $E$}{

  \For{each mini-batch $v_{\text{data}}$}{

    $T_t \leftarrow e^{\lambda_t} + \kappa \bar{\Delta F}_t$\;

    $h_{\text{data}} \sim \sigma((v_{\text{data}} W + b_h)/T_t)$\;

    Store $v_{\text{old}} \leftarrow v^{(0)}$\;

    \For{$k = 1$ \KwTo $K$}{
      $h^{(k)} \sim \sigma((v^{(k-1)} W + b_h)/T_t)$\;
      $v^{(k)} \sim \sigma((h^{(k)} W^\top + b_v)/T_t)$\;
    }

    $v^{(0)} \leftarrow v^{(K)}$\;

    $r_t \leftarrow \frac{1}{n_v} \| v^{(K)} - v_{\text{old}} \|_0$\;

    $\Delta W \leftarrow 
    \frac{1}{B}
    (v_{\text{data}}^\top h_{\text{data}}
    - v^{(K)\top} h^{(K)})
    - \lambda_w W$\;

    $W \leftarrow W + \eta \Delta W$\;

    Update $b_v, b_h$ analogously\;

    $\Delta F_t \leftarrow
    \mathbb{E}_{\text{data}}[F_{T_t}(v)]
    - \mathbb{E}_{\text{model}}[F_{T_t}(v)]$\;
  }

  $c_{t+1} \leftarrow (1-\alpha)c_t + \alpha r_t$\;

  $\lambda_{t+1}
  \leftarrow
  \lambda_t
  - \eta_\lambda (r_t - c_t)$\;

  $\bar{\Delta F}_{t+1}
  \leftarrow
  \bar{\Delta F}_t
  +
  \frac{1}{t}
  (\Delta F_t - \bar{\Delta F}_t)$\;
}
\end{algorithm}

At the beginning of training, the thermodynamic condition is initialized as $\lambda_0 = 0$ and $\bar{\Delta F}_0 = 0$; accordingly, the initial temperature is $T_0 = 1$. After completion of epoch $t$, the updated thermodynamic state $(\lambda_{t+1}, \bar{\Delta F}_{t+1})$ is propagated to the next epoch. The sampling temperature is not reset between epochs; instead, it evolves continuously as an endogenous state variable. Each epoch inherits the effective temperature from the previous one, and the temperature is expected to be transformed by the RBM's feedback dynamics.

\subsection{Code Availability and Reproducibility}

The complete source code used to generate all experimental results is publicly available on Zenodo (\href{https://doi.org/10.5281/zenodo.18778493}{10.5281/zenodo.18778493}), with the corresponding development repository hosted on GitHub (\href{https://github.com/cagasolu/sr-trbm}{github.com/cagasolu/sr-trbm}). The implementation includes deterministic seed control, logging of thermal diagnostics, and full AIS evaluation routines to enable independent replication of all reported metrics.

\subsection{Dataset and Preprocessing}

The proposed model is evaluated on the MNIST dataset (OpenML v1), consisting of 60,000 training samples and 10,000 test samples of handwritten digits (28\,$\times$\,28 grayscale images). Pixel intensities were normalized to the interval $[0,1]$ and subsequently binarized using a threshold at $0.5$:

\begin{equation}
x > 0.5 \rightarrow 1,
\qquad
x \le 0.5 \rightarrow 0.
\end{equation}

This preprocessing secures agreement with the Bernoulli–Bernoulli RBM formulation used throughout the study. On the other hand, in the experiments, bias decay was set to $0$ as part of the general optimization configuration.

\subsection{Model Design and Training Configuration}

All compared models share an identical architecture and optimization configuration:

\begin{itemize}[leftmargin=*, itemsep=2pt, label=--]
\item Visible units: 784
\item Hidden units: 512
\item Weight initialization: $\mathcal{N}(0, 0.05)$
\item Optimization algorithm: Persistent Contrastive Divergence (PCD-1)
\item Learning rate: $5 \times 10^{-4}$
\item Batch size: 128
\item Training epochs: 400
\item Weight decay: $10^{-4}$
\end{itemize}

Persistent negative-phase chains were maintained across minibatches to stabilize gradient estimation. The only difference among experimental conditions lies in the temperature control strategy:

\begin{itemize}[leftmargin=*, itemsep=2pt, label=--]
\item Fixed $\langle T = 1 \rangle$
\item Frozen $\langle T = T^{*} \rangle$ (manually tuned constant)
\item Adaptive (Self-Regulated Thermodynamic RBM)
\end{itemize}

In the adaptive setting, temperature evolves via a coupled micro–macro feedback mechanism that combines flip-rate regulation and Ces\`aro-averaged free-energy gap tracking. No manual model parameter setting was performed to match likelihood outcomes; temperature dynamics develop endogenously from the feedback process.

\subsection{Partition Function Estimation}

Log-partition values were estimated using fine-grained Annealed Importance Sampling (AIS) with:

\begin{itemize}[leftmargin=*, itemsep=2pt, label=--]
\item 3900 independent AIS chains
\item 7200 intermediate inverse-temperature distributions
\end{itemize}

This fine-grained annealing schedule was picked to ensure stable normalization estimates and reliable computation of Effective Sample Size (ESS).

\subsection{Assessment Criteria}

Model comparison is conducted along three complementary axes:

\begin{itemize}[leftmargin=*, itemsep=2pt, label=--]
\item \textbf{Likelihood-based evaluation:} Test log-likelihood (AIS-normalized) and pseudo-likelihood.
\item \textbf{Sampling efficiency diagnostics:} AIS Effective Sample Size (ESS), AIS log-weight variance, MCMC integrated autocorrelation time, and MCMC ESS.
\item \textbf{Generative characteristics:} Reconstruction error (MSE), pixel entropy, pairwise Hamming diversity, and the empirical mean image.
\end{itemize}

All models were trained under identical architectural and optimization settings. Consequently, observed performance differences are attributable directly to the thermodynamic regulation mechanism rather than to auxiliary hyperparameter effects.

\subsection{Empirical Findings}

Extra seed-level statistics, image creation and the adaptive setting graphs are registered in Appendix (\ref{app:seed_results}) and (\ref{app:image_results}), respectively.

\begin{table}[H]
\centering
\caption{Mean Performance Across Temperature Strategies (Higher log-likelihood and AIS ESS indicate better generative modeling and normalization quality.)}
\label{tab:mean-performance}
\fontsize{10pt}{12pt}\selectfont
\begin{tabular*}{\textwidth}{@{\extracolsep{\fill}}lrrr}
\toprule
Model & Test Log-Likelihood ($\uparrow$) & Recon. MSE ($\downarrow$) & AIS ESS ($\uparrow$) \\
\midrule
Fixed $\langle T\!=\!1\rangle$ 
& -714.29 $\pm$ 16.13 
& 0.016797 
& 65.23 \\

Fixed $\langle T\!=\!T^{*}\rangle$ 
& -689.39 $\pm$ 8.41 
& 0.016093 
& 65.82 \\

Adaptive 
& \textbf{-684.56 $\pm$ 9.22} 
& \textbf{0.016073} 
& \textbf{310.97} \\

\bottomrule
\end{tabular*}
\end{table}

Table~\ref{tab:mean-performance} summarizes the mean performance among temperature strategies. The Adaptive model obtains the highest test log-likelihood and the lowest reconstruction error among all configurations. More notably, it yields a substantially larger AIS Effective Sample Size (ESS), signifying enhanced normalization stability and more reliable partition function estimation. While reconstruction differences are numerically small, the gain in ESS suggests that adaptive thermodynamic regulation primarily enhances sampling efficiency instead of merely improving surface-level reconstruction precision.

\begin{table}[H]
\centering
\caption{Effect Size Analysis Across Temperature Strategies (Cohen's $d$)}
\label{tab:effect-sizes}
\fontsize{10pt}{12pt}\selectfont
\begin{tabular*}{\textwidth}{@{\extracolsep{\fill}}lrrr}
\toprule
Comparison 
& Test LL $d$ 
& Recon. MSE $d$ 
& Log AIS ESS $d$ \\
\midrule

Adaptive $\succ$ $\langle T\!=\!1\rangle$
& \textbf{1.47} 
& \textbf{12.4} 
& \textbf{2.11} \\

Adaptive $\succ$ $\langle T\!=\!T^{*}\rangle$ 
& 0.34     
& 0.21 
& \textbf{1.01} \\

$\langle T\!=\!T^{*}\rangle$ $\succ$ $\langle T\!=\!1\rangle$
& 1.87 
& 10.1 
& 0.43 \\

\bottomrule
\end{tabular*}
\end{table}

Table~\ref{tab:effect-sizes} reports standardized effect sizes computed using Cohen’s $d$ \cite{cohen1988statistical}. The largest effects are observed in Log AIS ESS, where the Adaptive strategy demonstrates great improvements over both fixed-temperature baselines. In contrast, effect sizes for reconstruction error are comparatively small, supporting the interpretation that the primary benefit of adaptive regulation lies in improved normalization dynamics instead of reconstruction exactness.

\begin{table}[H]
\centering
\caption{Bayesian Bootstrap Comparison on Log AIS ESS (ROPE defined as $\pm 2$ in log ESS space.)}
\label{tab:bayesian-bootstrap}
\fontsize{10pt}{12pt}\selectfont
\begin{tabular*}{\textwidth}{@{\extracolsep{\fill}}lrrr}
\toprule
Metric 
& Adaptive $\succ$ $\langle T\!=\!1\rangle$ 
& Adaptive $\succ$ $\langle T\!=\!T^{*}\rangle$ 
& $\langle T\!=\!T^{*}\rangle$ $\succ$ $\langle T\!=\!1\rangle$ \\
\midrule

$\Delta$ Mean 
& 3.01 
& 1.98 
& 1.03 \\

SD 
& 0.50 
& 0.69 
& 0.85 \\

95\% HDI 
& [2.05, 3.94] 
& [0.73, 3.37] 
& [-0.62, 2.62] \\

$P(\Delta>0)$ 
& 1.000 
& 1.000 
& 0.876 \\

ROPE ($|\Delta|<2$) 
& 0.031 
& 0.543 
& 0.865 \\

$BF_{10}$ 
& $6.5\times10^{7}$ 
& 61 
& 2.08 \\

Fold $\exp(\Delta)$ 
& $20.3\times$ 
& $7.2\times$ 
& $2.8\times$ \\

\bottomrule
\end{tabular*}
\end{table}

Table~\ref{tab:bayesian-bootstrap} summarizes the results of Bayesian bootstrap comparisons \cite{rubin1981bayesian}, where Bayesian evidence is determined by Bayes factors \cite{kass1995bayes} and practical equivalence is evaluated using the ROPE framework \cite{kruschke2013bayesian}. The Adaptive configuration is associated with overwhelming Bayesian evidence relative to the fixed $T=1$ baseline and with strong evidence in comparison to the tuned $T^{*}$ strategy, particularly in terms of normalization efficiency. The low ROPE probabilities further suggest that these improvements are unlikely to be practically negligible.

\section{Summary of Contributions}

When we compare the two distinct experiments in Table~\ref{tab:rbm_comparison}, the limitations of the classical RBM become more apparent. The classical formulation implicitly assumes that sampling takes place in a valid regime throughout training. It applies updates directly, regardless of how the sampler is actually performing at that moment. This observation suggests that reconstruction performance alone may not reliably reflect the stochastic consistency of the sampler.

\begin{table}[H]
\centering
\fontsize{9}{12}\selectfont
\caption{Comparison between the classical RBM and the proposed self-regulated RBM.}
\label{tab:rbm_comparison}
\renewcommand{\arraystretch}{1.25}

\adjustbox{max width=\textwidth}{
\begin{tabular}{l l l}
\toprule
{Aspect} &
{Classical RBM} &
{Self-Regulated RBM} \\
\midrule

Model interpretation
& Static energy-based model
& Dynamic, feedback-controlled stochastic system \\

Temperature ($T$)
& Fixed external hyperparameter
& Endogenous state variable with adaptive dynamics \\

Thermal equilibrium
& Implicitly assumed
& Explicitly monitored and operationally maintained \\

Gibbs sampling
& Unregulated alternating Gibbs updates
& Sampler-aware Gibbs dynamics \\

Sampler diagnostics
& Not defined
& Flip-rate and energy-fluctuation monitoring \\

Mixing behavior
& Uncontrolled and problem-dependent
& Actively regulated through feedback \\

Learning updates
& Applied uniformly at each iteration
& Applied under regulated stochastic regime \\

Learning stability
& Sensitive to $T$ and CD-$k$
& Designed to promote improved stabilization during training \\

Temperature scheduling
& Fixed or manually annealed
& Automatically adjusted via sampler feedback \\

Energy consistency
& Not explicitly enforced
& Energy-consistent reconstruction dynamics \\

Reconstruction mechanism
& Mean-field or heuristic mapping
& Deterministic, energy-consistent mapping \\

Failure modes
& Freezing or excessive stochasticity
& Dynamically corrected during training \\

Physical interpretation
& Idealized thermodynamic analogy
& Operational thermodynamic system \\

Role of RBM
& Approximate generative model
& Controlled stochastic physical system \\

Extensibility
& Limited by fixed dynamics
& Naturally extensible to deeper and adaptive models \\

\bottomrule
\end{tabular}
}
\end{table}

\section{Discussion and Conclusion}
Conceptually, the proposed framework reframes finite-time Gibbs training as a regulated non-equilibrium dynamical process rather than a static equilibrium approximation. In this view, the stochastic regime is explicitly monitored and adaptively controlled through thermodynamic feedback. We contend that the resulting regulatory principle may extend to wider categories of energy-based models trained with short-run stochastic approximations. In this context, we regard the Restricted Boltzmann Machine as a canonical, mathematically tractable setting for examining thermodynamic mechanisms in deep generative models.

Our study shows that fixed-temperature finite-time Gibbs training can induce thermodynamic degeneracy and, in the absence of sufficient regularization, drive linear parameter drift. Accordingly, structural stability under fixed-temperature finite-time Gibbs training remains inherently trajectory-dependent in nonconvex energy-based models.

The proposed framework manages this instability by elevating temperature from an external hyperparameter to an endogenous dynamical state variable. The resulting closed-loop system couples parameter evolution, stochastic transition activity, and thermodynamic feedback. Under standard local Lipschitz assumptions and a two-time-scale separation regime, the thermodynamic subsystem admits a locally exponentially stable operating point.

The stabilization system operates at both microscopic and macroscopic scales within the analyzed regime. The microscopic feedback rule regulates the empirical flip rate, guaranteeing sustained stochastic mixing, while the macroscopic Cesàro-averaged free-energy correction constrains cumulative energy imbalance and limits long-horizon thermodynamic drift. Together, these components obstruct the pathway from effective-field growth to conductance collapse and freezing-induced parameter drift within the regulated regime.

Empirically, adaptive thermodynamic regulation substantially improves normalization stability. The gains in ESS indicate that the primary benefit does not lie in reconstruction precision per se, but in enhanced sampling reliability and more stable partition-function estimation, consistent with recent examinations of short-run MCMC training in modern energy-based models \cite{nijkamp2020anatomy}.

From this perspective, thermal equilibrium is treated not as a fixed assumption but as an operational condition that may be locally maintained through feedback. This interpretation bridges finite-time Gibbs training \cite{lecun2022path} with dynamical-systems views of neural network equilibrium \cite{haber2018stable}.

While global asymptotic convergence is not claimed, the analysis shows a structural mechanism that increases robustness during finite-time stochastic training. The freezing phenomenon develops from the existence of admissible diverging trajectories, which preclude trajectory-independent structural guarantees in nonconvex energy-based models. Subsequent research should characterize conditions for global stability of the coupled parameter–temperature dynamics, plus extend thermodynamic regulation to deeper, continuous-state energy-based architectures.

\bibliographystyle{plain}

\appendix
\section{Seed-Level Experimental Results}
\label{app:seed_results}

This appendix provides complete seed-level performance measures and descriptive statistics for all temperature-control strategies evaluated in the main experiments. While the main text reports aggregate comparisons and statistical analyses, the present section documents cross-seed variability in likelihood, reconstruction precision, effective inverse temperature, and normalization diagnostics. These results assert full transparency of experimental behavior throughout independent random initializations. In the experimental execution, we set $\boldsymbol{\phi} = 1$, corresponding to pure integral feedback. Although the theoretical evaluation grants $0 < \boldsymbol{\phi} \le 1$, empirical stability was achieved through sufficiently small feedback gain $\eta_\lambda$.

\subsection{SR-TRBM Adaptive Setting}

Table~\ref{tab:adaptive_seed} reports individual seed-level metrics for the adaptive thermodynamic configuration. Across random initializations, the effective inverse temperature is tightly concentrated, showing stable regulation of the stochastic regime. In contrast, normalization diagnostics such as the AIS Effective Sample Size (ESS) exhibit moderate dispersion, reflecting variability in the quality of partition-function estimation, while staying substantially higher on average than fixed-temperature baselines.

\begin{table}[H]
\centering
\caption{Adaptive Temperature Results Across Seeds}
\label{tab:adaptive_seed}
\fontsize{10pt}{12pt}\selectfont
\begin{tabular*}{\textwidth}{@{\extracolsep{\fill}}c cc ccc}
\toprule
Seed 
& Test Log-Likelihood 
& Recon. MSE 
& $\beta^{\mathrm{norm}} = \|\theta\|/T$
& AIS ESS 
& AIS Log-Var \\
\midrule
1   & -675.912 & 0.015972 & 71.508 & 651.00 & 2.31 \\
2   & -684.010 & 0.016050 & 71.522 & 97.00 & 2.29 \\
3   & -685.915 & 0.016101 & 71.462 & 123.00 & 2.35 \\
17  & -681.731 & 0.016088 & 71.511 & 472.00 & 2.27 \\
128  & -705.228 & 0.016153 & 71.420 & 6.00  & 2.60 \\
360  & -680.108 & 0.016100 & 71.391 & 541.00 & 2.24 \\
1000 & -679.030 & 0.016045 & 71.529 & 281.78 & 2.33 \\
\bottomrule
\end{tabular*}
\end{table}

Table~\ref{tab:adaptive_desc} summarizes descriptive statistics across seeds. The adaptive configuration shows lower variance in likelihood and reconstruction error than the fixed $T=1$ baseline, along with a significantly higher mean AIS ESS. This behavior supports the claim that endogenous temperature control enhances normalization stability without damaging reconstruction precision.

\begin{table}[H]
\centering
\caption{Descriptive Statistics for Adaptive Temperature}
\label{tab:adaptive_desc}
\fontsize{10pt}{12pt}\selectfont
\begin{tabular*}{\textwidth}{@{\extracolsep{\fill}}lrrl}
\toprule
Metric & Mean & Std & 95\% CI \\
\midrule
Test Log-Likelihood & -684.56 & 9.22 & [-693.08, -676.04] \\
Reconstruction MSE & 0.016073 & 0.000060 & [0.016017, 0.016129] \\
Effective $\beta$  & 71.478 & 0.054 & [71.43, 71.53] \\
AIS ESS       & 310.97 & 247.61 & [20.7, 336.5] \\
Log AIS ESS     & 4.43  & 1.51  & [3.03, 5.82] \\
\bottomrule
\end{tabular*}
\end{table}

\subsection{SR-TRBM Fixed Setting}

Table~\ref{tab:fixed1_seed} puts forth seed-level metrics for the classical fixed-temperature setting with $T=1$. Compared with the adaptive configuration, the fixed-temperature regime exhibits greater dispersion in AIS ESS and AIS log-weight variance.

\begin{table}[H]
\centering
\caption{Seed-Level Performance Metrics for Fixed $\langle T\!=\!1\rangle$}
\label{tab:fixed1_seed}
\fontsize{10pt}{12pt}\selectfont
\begin{tabular*}{\textwidth}{@{\extracolsep{\fill}}c cc ccc}
\toprule
Seed 
& Test Log-Likelihood 
& Recon. MSE 
& $\beta^{\mathrm{norm}} = \|\theta\|/T$
& AIS ESS 
& AIS Log-Var \\
\midrule
1   & -693.716 & 0.016702 & 68.988 & 420.71 & 2.18 \\
2   & -702.772 & 0.016771 & 69.010 & 11.52 & 2.89 \\
3   & -716.324 & 0.016807 & 68.981 & 2.38  & 7.83 \\
17  & -720.409 & 0.016832 & 69.050 & 7.33  & 9.51 \\
128  & -744.530 & 0.016856 & 68.934 & 1.04  & 23.79 \\
360  & -701.154 & 0.016820 & 68.870 & 9.55  & 2.72 \\
1000 & -721.102 & 0.016790 & 69.039 & 4.06  & 11.64 \\
\bottomrule
\end{tabular*}
\end{table}

Descriptive statistics in Table~\ref{tab:fixed1_desc} further illustrate this instability. The higher standard deviation in log-likelihood and the large variance in AIS ESS reflect the structural weakness of fixed-temperature finite-time training identified in the theoretical evaluation.

\begin{table}[H]
\centering
\caption{Descriptive Statistics for Fixed $\langle T\!=\!1\rangle$}
\label{tab:fixed1_desc}
\fontsize{10pt}{12pt}\selectfont
\begin{tabular*}{\textwidth}{@{\extracolsep{\fill}}lrrl}
\toprule
Metric & Mean & Std & 95\% CI \\
\midrule
Test Log-Likelihood & -714.29 & 16.13 & [-729.19, -699.38] \\
Reconstruction MSE & 0.016797 & 0.000050 & [0.016751, 0.016843] \\
Effective $\beta$  & 68.982 & 0.065 & [68.92, 69.04] \\
AIS ESS       & 65.23 & 155.05 & [1.80, 63.5] \\
Log AIS ESS     & 2.37  & 1.97  & [0.59, 4.15] \\
\bottomrule
\end{tabular*}
\end{table}

\subsection{SR-TRBM Frozen Setting}

Table~\ref{tab:frozen_seed} reports seed-level results for the tuned constant-temperature configuration. While performance improves relative to $T=1$, cross-seed variability in normalization diagnostics remains substantial. The effective inverse temperature is the same by construction, yet AIS ESS does not consistently reach the levels achieved by the adaptive strategy.

\begin{table}[H]
\centering
\caption{Fixed $\langle T \!=\!T^{*}\rangle$ Results Across Seeds}
\label{tab:frozen_seed}
\fontsize{10pt}{12pt}\selectfont
\begin{tabular*}{\textwidth}{@{\extracolsep{\fill}}c cc ccc}
\toprule
Seed 
& Test Log-Likelihood 
& Recon. MSE 
& $\beta^{\mathrm{norm}} = \|\theta\|/T$
& AIS ESS 
& AIS Log-Var \\
\midrule
1   & -682.163 & 0.016007 & 71.708 & 102.08 & 2.78 \\
2   & -707.188 & 0.016116 & 71.744 & 1.40  & 11.22 \\
3   & -686.139 & 0.016073 & 71.645 & 114.58 & 2.79 \\
17  & -688.951 & 0.016093 & 71.729 & 114.17 & 2.47 \\
128  & -687.254 & 0.016172 & 71.606 & 7.72  & 3.57 \\
360  & -692.620 & 0.016117 & 71.599 & 3.06  & 3.44 \\
1000 & -681.446 & 0.016071 & 71.721 & 113.74 & 2.97 \\
\bottomrule
\end{tabular*}
\end{table}

Table~\ref{tab:frozen_desc} presents the corresponding descriptive statistics. While tuning the temperature reduces the likelihood dispersion relative to $T=1$, normalization quality still falls short of the adaptive thermodynamic regime. Static temperature tuning, therefore, does not fully replace endogenous stochastic regulation.

\begin{table}[H]
\centering
\caption{Descriptive Statistics for Fixed $\langle T\!=\!T^*\rangle$}
\label{tab:frozen_desc}
\fontsize{10pt}{12pt}\selectfont
\begin{tabular*}{\textwidth}{@{\extracolsep{\fill}}lrrl}
\toprule
Metric & Mean & Std & 95\% CI \\
\midrule
Test Log-Likelihood & -689.39 & 8.41 & [-697.17, -681.61] \\
Reconstruction MSE & 0.016093 & 0.000060 & [0.016037, 0.016149] \\
Effective $\beta$  & 71.68 & 0.06 & [71.63, 71.73] \\
AIS ESS       & 65.82 & 55.30 & [6.05, 97.6] \\
Log AIS ESS     & 3.19 & 1.87 & [1.80, 4.58] \\
\bottomrule
\end{tabular*}
\end{table}

\section{Image-Level Experimental Results}
\label{app:image_results}

To further evaluate generative characteristics beyond likelihood and normalization metrics, we report pixel entropy, pairwise Hamming diversity, and the mean $\ell_2$ distance of synthesized samples. These measures are computed for representative seeds throughout all temperature-control strategies. Representative image generations together with the adaptive-regime thermodynamic diagnostics for Seed~1 are shown in Figures~\ref{fig:fixed_seed1} and~\ref{fig:energy_dynamics}.

\begin{table}[H]
\centering
\caption{Generative diagnostics for selected seeds under Adaptive settings, Fixed $\langle T=1\rangle$, and Frozen $\langle T=T^*\rangle$ settings.}
\label{tab:generative_diagnostics_full}
\small
\begin{tabular*}{\textwidth}{@{\extracolsep{\fill}}lcccc}
\toprule
Setting & Seed & Pixel Entropy & Diversity & Mean $\ell_2$ \\
\midrule
Adaptive & 1  & 0.2948 & 0.1974 & 2.5160 \\
Adaptive & 128 & 0.2537 & 0.1646 & 1.3146 \\
\midrule
Fixed $\langle T=1\rangle$ & 1  & 0.2547 & 0.1661 & 1.9007 \\
Fixed $\langle T=1\rangle$ & 128 & 0.2540 & 0.1651 & 1.8280 \\
\midrule
Frozen $\langle T=T^*\rangle$ & 1  & 0.2642 & 0.1737 & 1.9898 \\
Frozen $\langle T=T^*\rangle$ & 128 & 0.2082 & 0.1288 & 2.4859 \\
\bottomrule
\end{tabular*}
\end{table}

Across strategies, the adaptive configuration displays the highest entropy and diversity for Seed 1, indicating stronger stochastic exploration. The fixed $T=1$ regime produces tightly clustered entropy and diversity values across seeds, implying a more constrained generative regime. The frozen temperature setting shows asymmetric behavior: while Seed 1 achieves moderate entropy and diversity, Seed 128 displays noticeably lower entropy, indicating sensitivity upon initialization. These observations are consistent with the conceptual analysis, which predicts that endogenous thermodynamic regulation stabilizes stochastic activity while preserving controlled exploration.

\clearpage

\begin{figure}[!t]
\centering
\includegraphics[width=1.0\textwidth]{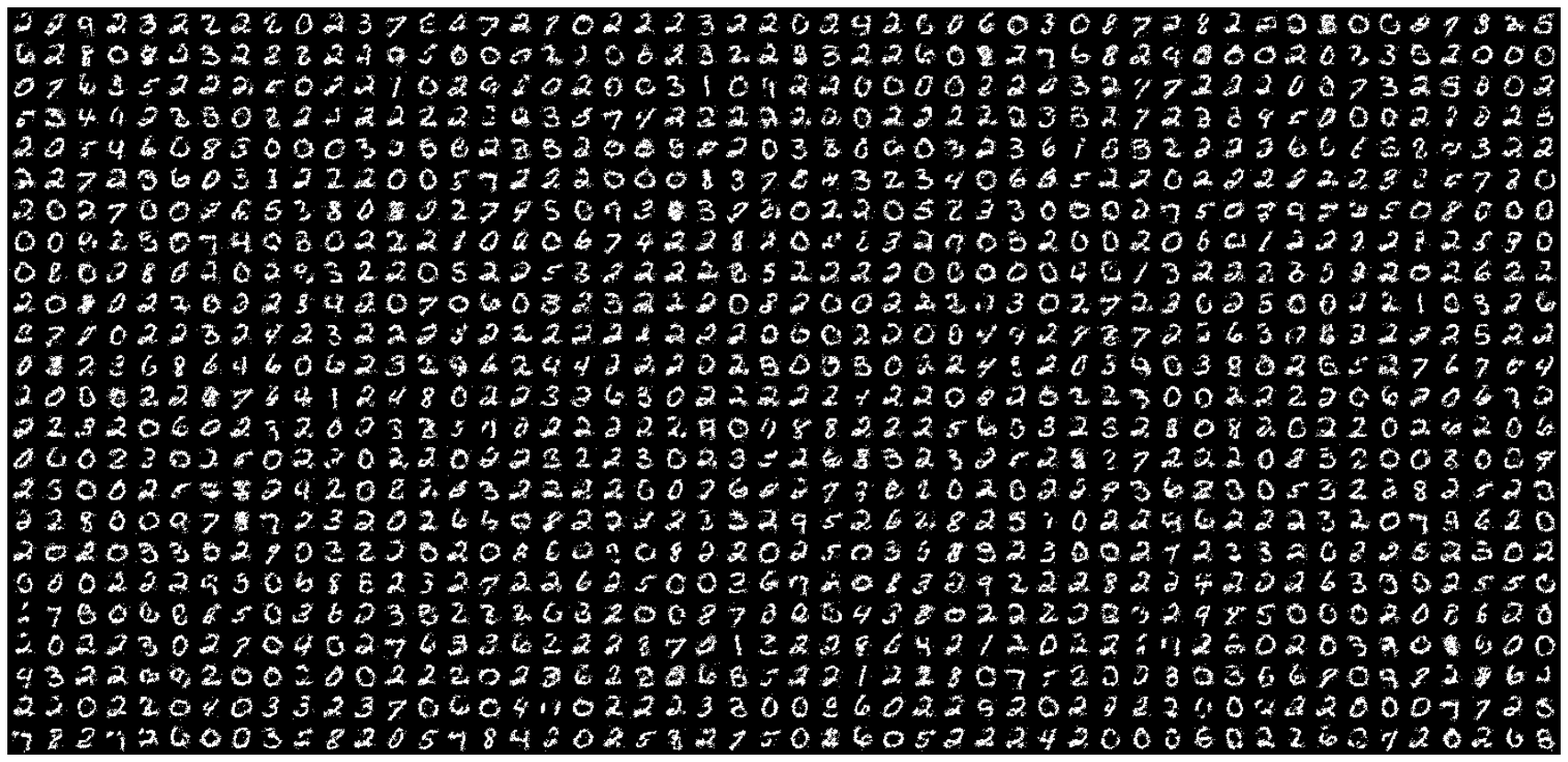}
\caption{Adaptive temperature-control regime (Seed 1).
\label{fig:adaptive_seed1}
Samples were obtained after 6000 steps of ensemble Gibbs sampling.
The generated configurations organize along a coherent correlation manifold. Digits containing multiple interdependent stroke structures show clearer structural persistence, signifying improved modeling of higher-order associative patterns within the regulated stochastic regime.}
\end{figure}

\clearpage

\begin{figure}[!t]
\centering
\includegraphics[width=1.0\textwidth]{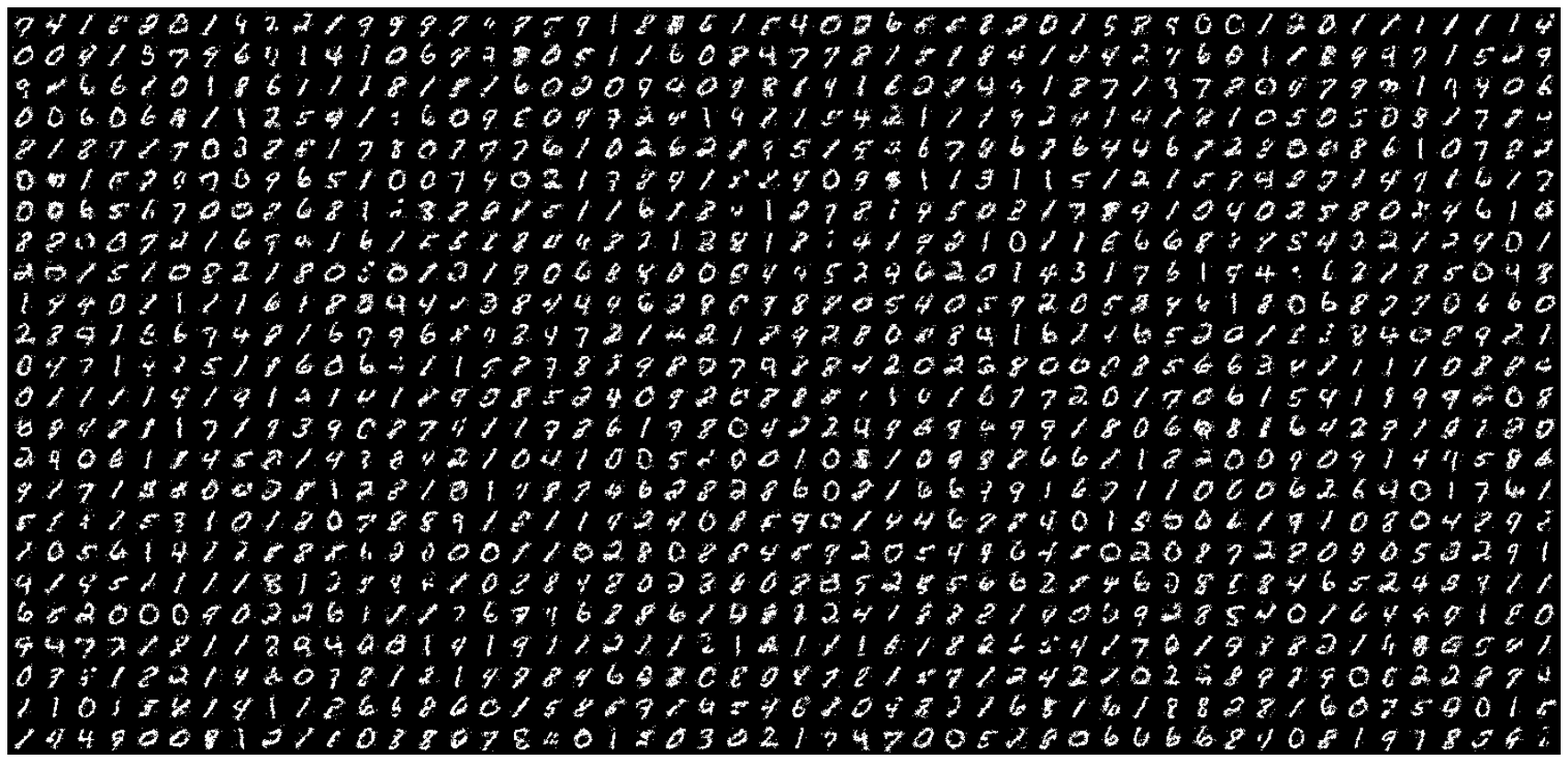}
\caption{Fixed temperature setting ($T=1$, Seed 1).
\label{fig:fixed_seed1}
Samples were obtained after 6000 steps of ensemble Gibbs sampling.
The distribution displays comparatively constrained generative diversity, with compressed correlation geometry and reduced variation within correlated stroke structures.}
\end{figure}

\clearpage

\begin{figure}[!t]
\centering
\includegraphics[width=1.0\textwidth]{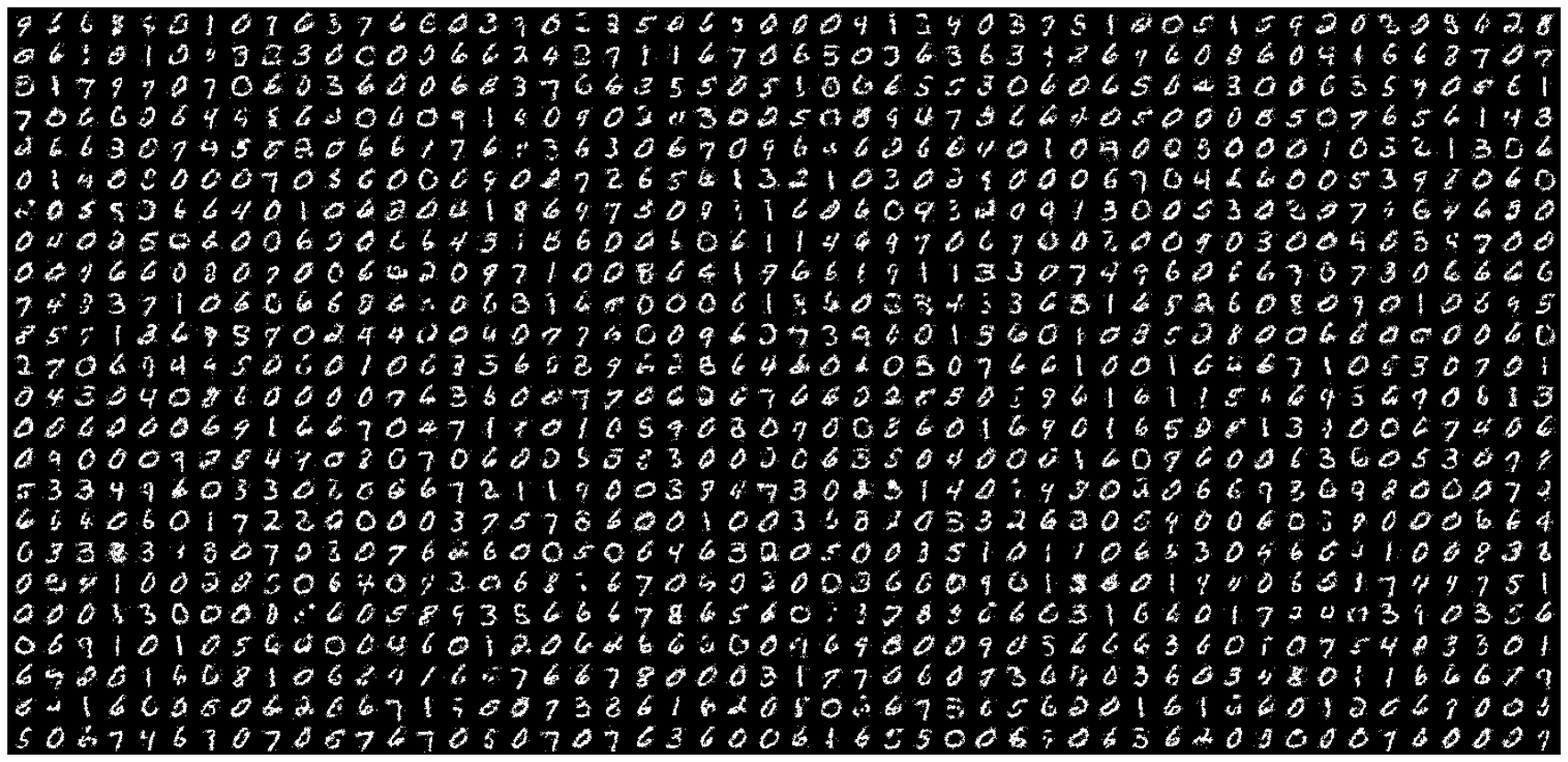}
\caption{Frozen temperature setting ($T = T^{\ast}$, Seed 1).
\label{fig:frozen_seed1}
Samples were obtained after 6000 steps of ensemble Gibbs sampling.
The generative configurations display sensitivity to the initialization and to the selection of fixed temperature, reflecting unstable correlation alignment and reduced architectural robustness.}
\end{figure}

\clearpage

\thispagestyle{empty}

\begin{figure}[p]
\centering
\includegraphics[width=\textwidth]{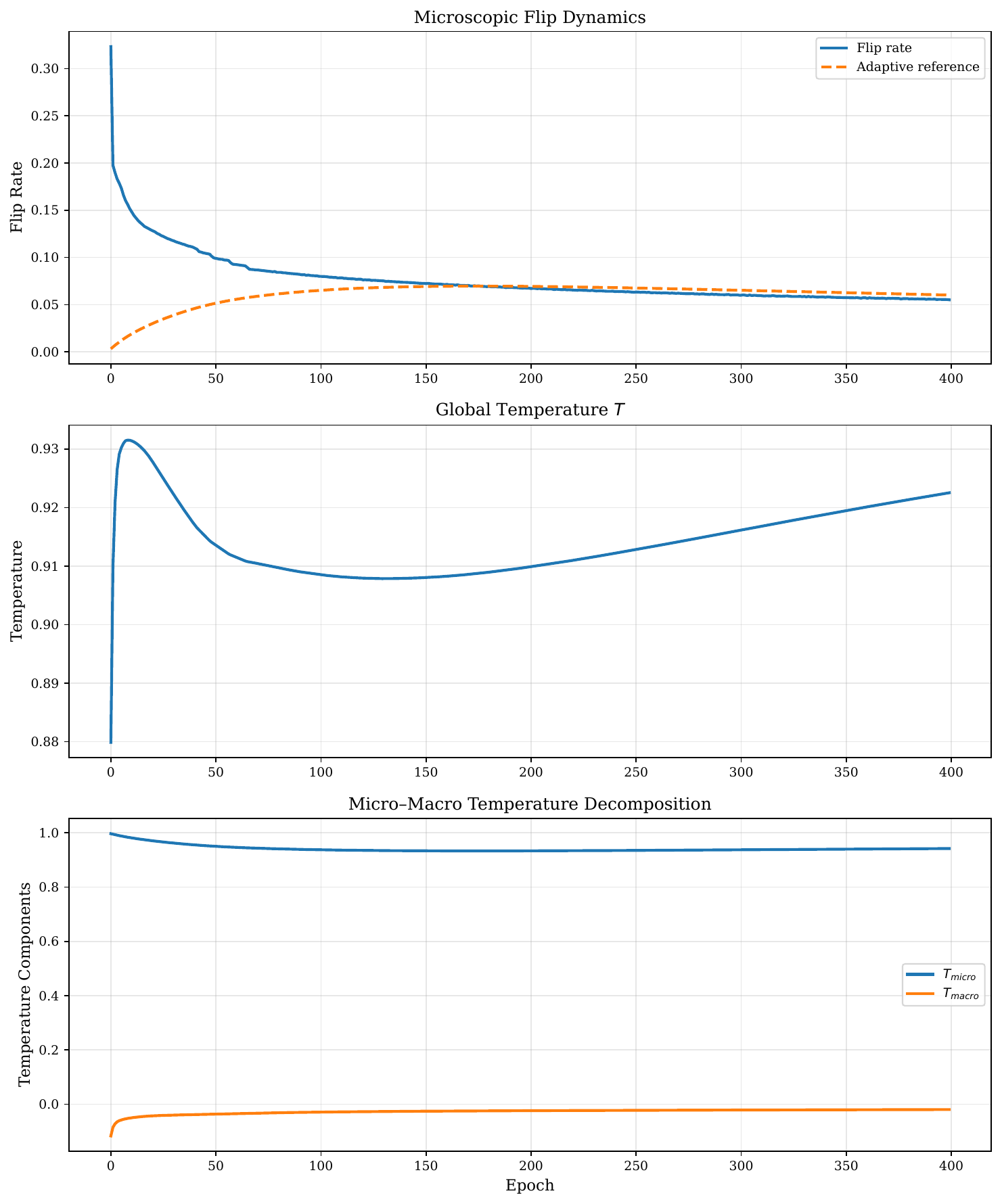}
\caption{
Microscopic flip dynamics and temperature evolution under endogenous regulation, according to adaptive setting, Seed is 1 for this graph and Figure~\ref{fig:energy_dynamics}. Top: flip rate with adaptive reference, illustrating progressive suppression of stochastic transitions. Middle: global temperature trajectory. Bottom: micro--macro temperature decomposition. The decay in flip activity reflects the controlled thermodynamic stabilization mechanism.}
\label{fig:core_dynamics}
\end{figure}

\clearpage

\thispagestyle{empty}

\begin{figure}[p]
\centering
\includegraphics[width=\textwidth]{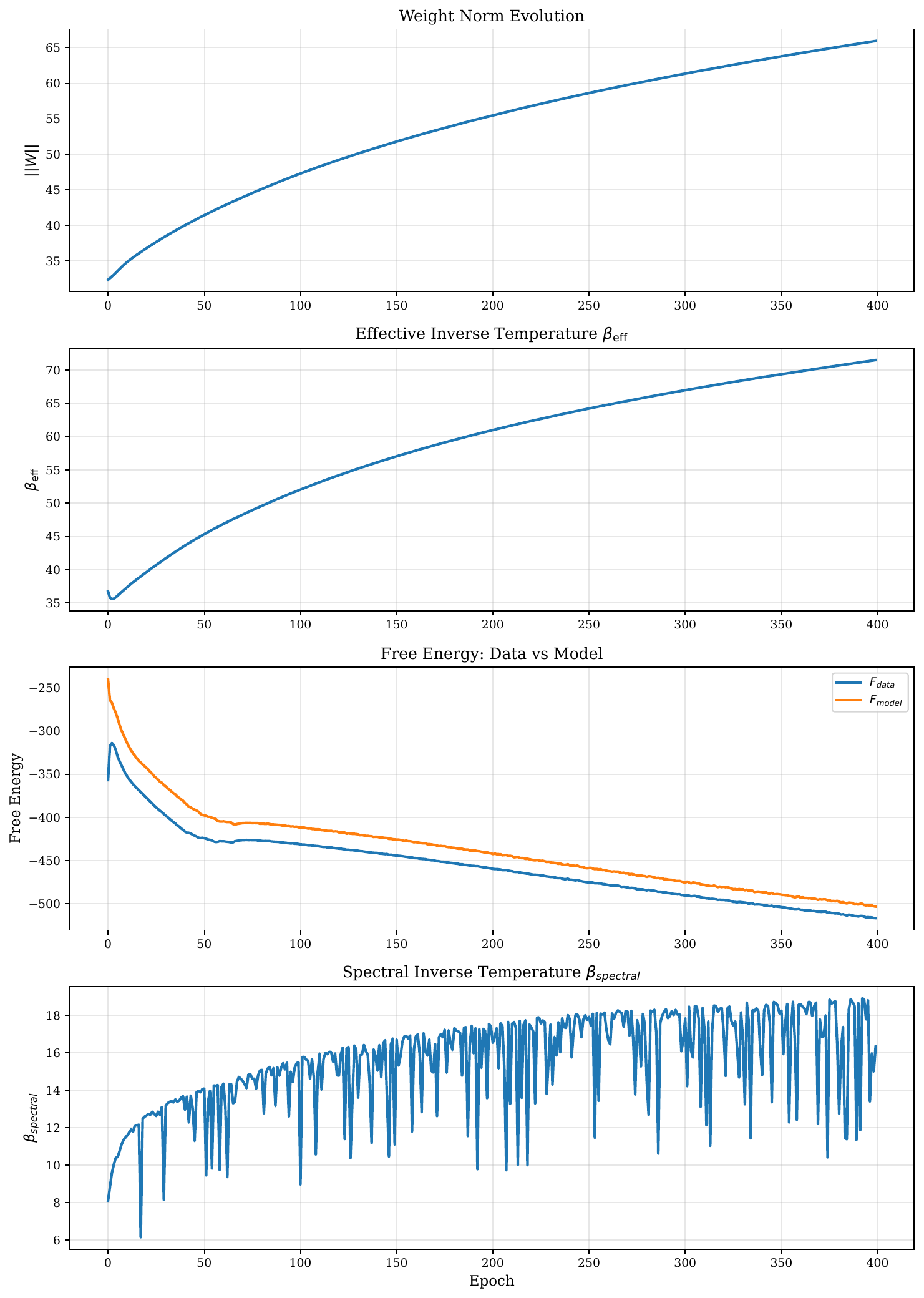}
\caption{
Energy-based diagnostics during training.
Top: weight norm evolution.
Second: effective inverse temperature $\beta_{\mathrm{eff}} = \|W_t\|_{F} / T_t$, where $\|\cdot\|_{F}$ denotes the Frobenius norm, capturing the macroscopic scaling of the energy landscape relative to temperature. This quantity is distinct from the field-level inverse temperature 
$\beta^{\mathrm{field}}$ appearing in Theorem~\ref{thm:freezing}. Third: free energy comparison between data and model. Bottom: spectral inverse temperature $\beta_{\mathrm{spectral}} = \|W_t\|_{2} / T_t$, where $\|\cdot\|_{2}$ denotes the spectral norm (largest singular value).}
\label{fig:energy_dynamics}
\end{figure}

\end{document}